\documentclass[10pt,journal,compsoc]{IEEEtran}

\usepackage{amsthm}
\usepackage{amsmath}
\usepackage{amssymb}
\usepackage{graphicx}
\usepackage{algorithm}
\usepackage{algorithmic}
\usepackage{dblfloatfix}
\usepackage{color}
\usepackage{bm}
\usepackage{subfigure}

\usepackage{hyperref}
\hypersetup{
    bookmarks=true,         
    unicode=false,          
    pdftoolbar=true,        
    pdfmenubar=true,        
    pdffitwindow=false,     
    pdftitle={Certificate},    
    pdfauthor={Dr. Harish Kumar},     
    pdfsubject={TEQIP certificates},   
    pdfcreator={Dr. Harish Kumar},   
    pdfproducer={},  
    pdfkeywords={Certificates,} {TEQIP} {Participation}, 
    pdfnewwindow=true,      
    colorlinks=false,       
    linkcolor=red,          
    citecolor=green,        
    filecolor=magenta,      
    urlcolor=cyan           
}

\usepackage{xspace}
\makeatletter
\DeclareRobustCommand\onedot{\futurelet\@let@token\@onedot}
\def\@onedot{\ifx\@let@token.\else.\null\fi\xspace}
 
\def\ie{{i.e}\onedot}

\makeatother

\DeclareMathOperator*{\subject}{s.t.}
\DeclareMathOperator*{\rank}{rank}
\DeclareMathOperator*{\sgn}{sgn}
\DeclareMathOperator*{\for}{for}

\newtheorem{innercustomthm}{Theorem}

\ifCLASSOPTIONcompsoc
  \usepackage[nocompress]{cite}
\else
  \usepackage{cite}
\fi
\ifCLASSINFOpdf
\else
\fi
\begin{document}
%
\title{Side Information for Face Completion: a Robust PCA Approach}
%
%
%
%

\author{Niannan~Xue,~\IEEEmembership{Student~Member,~IEEE,}
Jiankang~Deng,~\IEEEmembership{Student~Member,IEEE,} Shiyang~Cheng,~\IEEEmembership{Student~Member,IEEE,} Yannis~Panagakis,~\IEEEmembership{Member,IEEE,}
        and~Stefanos~Zafeiriou,~\IEEEmembership{Member,~IEEE}
\IEEEcompsocitemizethanks{\IEEEcompsocthanksitem N. Xue, J. Deng, S. Cheng, Y. Panagakis and S. Zafeiriou are with the Department of Computing,
Imperial College London, UK\protect\\
Corresponding author: Jiankang Deng, E-mail: j.deng16@imperial.ac.uk
\IEEEcompsocthanksitem S. Zafeiriou is also with Center for Machine Vision and Signal Analysis, University of Oulu, Finland.}
\thanks{Manuscript received April 19, 2005; revised August 26, 2015.}}

%
%

\markboth{Journal of \LaTeX\ Class Files,~Vol.~14, No.~8, August~2015}%
{Shell \MakeLowercase{\textit{et al.}}: Bare Advanced Demo of IEEEtran.cls for IEEE Computer Society Journals}
%



\IEEEtitleabstractindextext{%
\begin{abstract}

Robust principal component analysis (RPCA) is a powerful method for learning low-rank feature representation of various visual data. However, for certain types as well as significant amount of error corruption, it fails to yield satisfactory results; a drawback that can be alleviated by exploiting   domain-dependent prior knowledge or information. In this paper, we propose two models for the RPCA that take into account  such side information, even in the presence of missing values. We apply this framework to the task of UV completion which is widely used in pose-invariant face recognition. Moreover, we construct a generative adversarial network (GAN) to extract side information as well as subspaces. These subspaces not only assist in the recovery but also speed up the process in case of large-scale data. We quantitatively and qualitatively evaluate the proposed approaches through both synthetic data and five real-world datasets to verify their effectiveness.
\end{abstract}

}

\maketitle

\IEEEdisplaynontitleabstractindextext

%
\IEEEpeerreviewmaketitle

\ifCLASSOPTIONcompsoc
\IEEEraisesectionheading{\section{Introduction}\label{sec:introduction}}
\else
\section{Introduction}
\label{sec:introduction}
\fi

\begin{figure*}[t]
    \centering
    \includegraphics[width=0.98\linewidth]{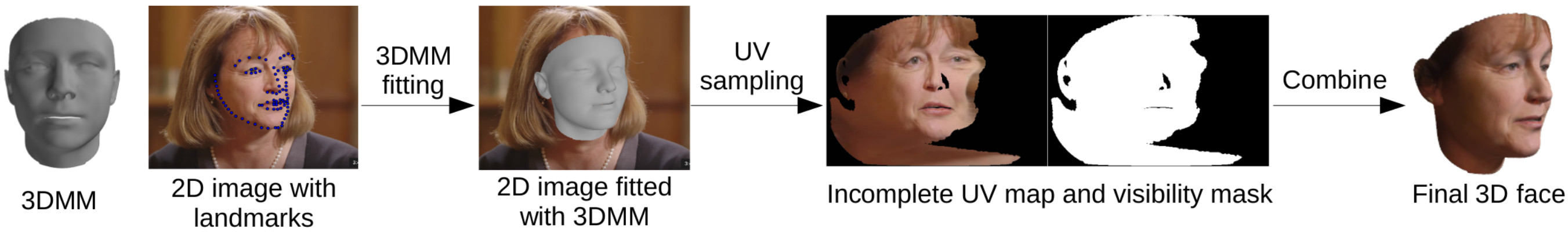}
    \caption{The procedure of getting the UV map from an arbitrary 2D image.}
    \label{fig:uv_demo}
\end{figure*}

\begin{figure}[b!]
\centering
\includegraphics[width=1\linewidth]{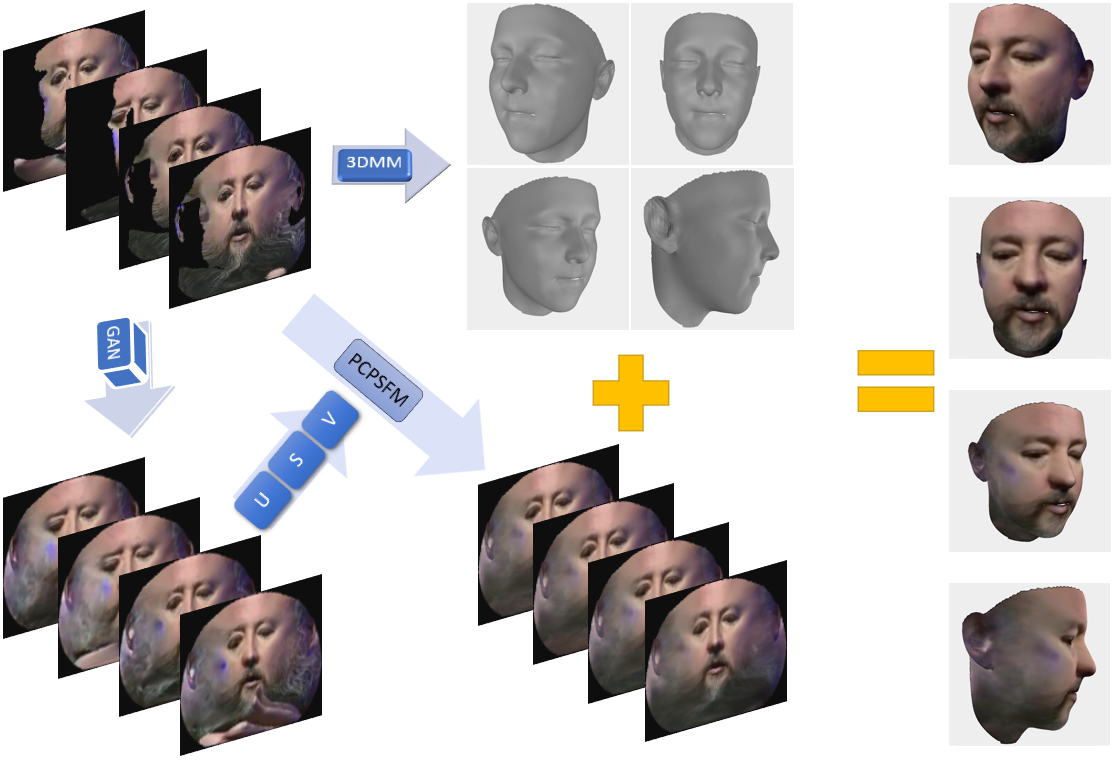}
\caption{Given an input sequence of incomplete UV maps, we extract the shape using 3DMM and perform preliminary completion using GAN. With the left subspace and side information provided by GAN, we then carry out PCPSFM to produce more refined completion results. After that, we attach the completed UV texture to the shape creating images at various poses for face recognition.}
\label{fig:diagram}
\end{figure}

\IEEEPARstart{U}{V} space embeds the manifold of a 3D face as a 2D contiguous atlas. Contiguous UV spaces are natural products of many 3D scanning devices and are often used for 3D Morphable Model (3DMM) construction~\cite{Booth14,3dmm_vetter,3dmm_revisited_PatelS09}. Although UV space by nature cannot be constructed from an arbitrary 2D image, a UV map may still be obtained by fitting a 3DMM to the image and sampling the corresponding texture~\cite{Booth17}. We illustrate this procedure in Figure~\ref{fig:uv_demo}. Unfortunately, due to the self occlusion of the face, those UV maps are always incomplete and are likely to miss facial parts that are informative. Once completed, this UV map, combined with the corresponding 3D face, is extremely useful, as it can be used to synthesise 2D faces of arbitrary poses. Then, we can probe image pairs of similar poses to improve recognition performance \cite{Cao17}. Hence, the success of pose-invariant face recognition relies on the quality of UV map completion.

Recovering UV maps from a sequence of related facial frames can addressed by employing  robust principal component analysis (RPCA) with missing data \cite{Shang14}. This is because self-occlusion at large poses leads to incomplete and missing data and imperfection in fitting leads to regional errors. Principal Component Pursuit (PCP) as proposed in \cite{Candes11,Chandrasekaran11} and its variants e.g., \cite{Aravkin14,Bao12,Cabral2013,Xu12,Zhou10} are popular algorithms to solve RPCA. PCP employs the nuclear norm and the $l_1$-norm (convex surrogates of the rank and sparsity constraints, respectively) in order to approximate the original $l_0$-norm regularised rank minimisation problem. Unfortunately, PCP operates in an isolated manner where domain-dependent prior knowledge\cite{Jiao15}, i.e., side information \cite{Ziv76}, is always ignored. Moreover, real-world visual data rarely satisfies the stringent assumptions imposed by PCP for exact recovery\cite{Candes08}. These call for a more powerful framework that can assimilate useful priors to alleviate the degenerate or suboptimal solutions of PCP.

It has already been shown that side information is propitious in the context of matrix completion \cite{Chiang15,Xu13} and compressed sensing \cite{Mota17}. Recently, \textit{noiseless} features have been capitalised on in the PCP framework\cite{KaiYang16,Sagonas14,liu10,liu17}. In particular, an error-free orthogonal column space was used to drive a person-specific facial deformable model\cite{Sagonas14}. And features also remove dependency on the row-coherence which is beneficial in the case of union of multiple subspaces\cite{liu10,liu17}. More generally, Chiang et al.\cite{KaiYang16} used both a column and a row space to recover only the weights of their interaction in a simpler problem. The main hindrance to the success of such methods is the need for a set of clean, noise-free data samples in order to determine the column and/or row spaces of the low-rank component. But there are no prescribed way to find them in practice.

On a separate note, rapid advances in neural networks for image inpainting offers an agglomeration of useful priors. Pathak et al.\cite{pathak2016} proposed context encoders with a re-construction and an adversarial loss to generate the contents for missing regions that comply with the neighbourhood. Yang et al.\cite{Yang2017} further improved inpainting with a multi-scale neural patch synthesis method. This approach is based on a joint optimisation of image content and texture constraints, which not only preserves contextual structures but also produces fine details. Li et al.\cite{Li2017} combined a reconstruction loss, two adversarial losses, and a semantic parsing loss to ensure genuineness and consistency of local-global contents. These methods are by no means definitive for the following reasons: (a) their masks are artificial and do not have semantical correspondence with a 3D face; (b) they do not allow missing regions over $50\%$ which is commonplace in our case.

This paper is based on our preliminary work\cite{Xue17} and extended to 1) the problem of UV completion and 2) to incorporate side information provided by generative adversarial networks. As such, we have extended PCP to take advantage of \textit{noisy} prior information aiming to realise better UV map reconstruction. We then perform pose-invariant face recognition experiments using the completed UV maps.
Experimental results indicate the superiority of our framework. The overall work flow is explicated in Fig.\ref{fig:diagram}. Our contributions are summarised as follows:
\begin{itemize}
    \item A novel convex program is proposed to use side information, which is a noisy approximation of the low-rank component, within the PCP framework. The proposed method
    is able to handle missing values while the developed optimization algorithm has convergence guarantees.
    \item Furthermore, we extend our proposed PCP model using side information to exploit prior knowledge regarding the column and row spaces of the low-rank component in a more general algorithmic framework.
    \item In the case of UV completion, we suggest the use of generative adversarial networks to provide subspace features and side information, leading to a seamless integration of deep learning into the robust PCA framework.
    \item We demonstrate the applicability and effectiveness of the proposed approaches on synthetic data as well as on facial image denoising, UV texture completion and pose-invariant face recognition experiments with both quantitative and qualitative evaluation.\\
\end{itemize}

The remainder of this paper is organised as follows. We discuss relevant literature in Section 2, while the proposed robust principal component analysis using side information with missing values (PCPSM) along with its extension that incorporates features (PCPSFM) is presented in Section 3. In Section 4, we first evaluate our proposed algorithms on synthetic and real-world data. Then we introduce GAN as a source of features and side information for the subject of UV completion. Finally, face recognition experiments are presented in the last subsection.

\textbf{\textit{Notations}} Lowercase letters denote scalars and uppercase letters denote matrices, unless otherwise stated. For norms of matrix $\mathbf{A}$, $\Vert \mathbf{A}\Vert_F$ is the Frobenius norm; $\Vert \mathbf{A}\Vert_*$ is the nuclear norm; and $\Vert \mathbf{A}\Vert_1$ is the sum of absolute values of all matrix entries. Moreover, $\langle \mathbf{A},\mathbf{B}\rangle$ represents tr($\mathbf{A}^T\mathbf{B}$) for real matrices $\mathbf{A},\mathbf{B}$. Additionally, $\mathbf{A}\circ\mathbf{B}$ symbolises element-wise multiplication of two matrices of the same dimension.

\section{Related work}

We discuss two different lines of research, namely low-rank recovery as well as image completion.

\subsection{Robust principal component analysis}
Suppose that there is a matrix $\mathbf{L}_0\in\mathbb{R}^{n_1\times n_2}$ with rank $r\ll$ min($n_1,n_2$) and a sparse matrix $\mathbf{E}_0\in\mathbb{R}^{n_1\times n_2}$ with entries of arbitrary magnitude. If we are provided with the observation matrix $\mathbf{X} =\mathbf{L}_0+\mathbf{E}_0$, RPCA aims to recover them by solving the following objective:
\begin{equation}
\label{eq:RPCA}
    \min_{\mathbf{L},\mathbf{E}}\rank(\mathbf{L})+\lambda\Vert\mathbf{E}\Vert_0\quad\subject\quad\mathbf{X}=\mathbf{L}+\mathbf{E},
\end{equation}
where $\lambda$ is a regularisation parameter.
However, (\ref{eq:RPCA}) cannot be readily solved because it is NP-hard. PCP instead solves the following convex surrogate:
\begin{equation}
\label{eq:PCP}
    \min_{\mathbf{L},\mathbf{E}}\Vert\mathbf{L}\Vert_*+\lambda\Vert\mathbf{E}\Vert_1\quad\subject\quad\mathbf{X}=\mathbf{L}+\mathbf{E},
\end{equation}
which, under mild conditions, is equivalent to (\ref{eq:RPCA}). There exist many efficient solvers for (\ref{eq:PCP}) and its applications include background modelling from surveillance video and removing shadows and specularities from face images.

One the first methods for incorporating dictionary was proposed in the context of subspace clustering\cite{liu10,liu17}. The LRR algorithm assumes that we have available an orthogonal column space $\mathbf{U}\in\mathbb{R}^{n_1\times d_1}$, where $d_1\le n_1$, and optimises the following:
\begin{equation}
\label{eq:LRR}
    \min_{\mathbf{K},\mathbf{E}}\quad\Vert \mathbf{K}\Vert_*+\lambda\Vert\mathbf{E}\Vert_1\quad\subject\quad\mathbf{X}=\mathbf{U}\mathbf{K}+\mathbf{E}.
\end{equation}
Given an orthonormal statistical prior of facial images, LRR can be used to construct person-specific deformable models from erroneous initialisations\cite{Sagonas14}.

A generalisation of the above was proposed as Principal Component Pursuit with Features (PCPF)\cite{KaiYang16} where further row spaces $\mathbf{V}\in\mathbb{R}^{n_2\times d_2}$, $d_2\le n_2$, were assumed to be available with the following objective:
\begin{equation}
\label{eq:PCPF}
    \min_{\mathbf{H},\mathbf{E}}\quad\Vert \mathbf{H}\Vert_*+\lambda\Vert\mathbf{E}\Vert_1\quad\subject\quad\mathbf{X}=\mathbf{U}\mathbf{H}\mathbf{V}^T+\mathbf{E}.
\end{equation}
There is a stronger equivalence relation between (\ref{eq:PCPF}) and (\ref{eq:RPCA}) than (\ref{eq:PCP}). The main drawback of the above mentioned models is that features need to be accurate and noiseless, which is not trivial in practical scenarios.

In the case of missing data, robust matrix recovery methods\cite{Chen13,Shang14} enhance PCP to deal with occlusions:
\begin{equation}
\label{eq:PCPM}
    \min_{\mathbf{L},\mathbf{E}}\Vert\mathbf{L}\Vert_*+\lambda\Vert\mathbf{E}\circ\mathbf{W}\Vert_1\quad\subject\quad\mathbf{X}=\mathbf{L}+\mathbf{E},
\end{equation}
where $\mathbf{W}$ is the matrix of binary occlusion masks. Its Jacobi-type update schemes can be implemented in parallel and hence are attractive for solving large-scale problems.

\subsection{Image completion neural networks}
Recent advances in convolutional neural networks (CNN) also show great promises in visual feature learning. Context encoders (CE)\cite{pathak2016} use a encoder-decoder pipeline where the encoder takes an input image with missing regions producing a latent feature representation and the decoder takes the feature representation generating the missing image content. CE uses a joint loss function:
\begin{equation}
    \mathcal{L}=\lambda_{rec}\mathcal{L}_{rec}+\lambda_{adv}\mathcal{L}_{adv},
\end{equation}
where $\mathcal{L}_{rec}$ is the reconstruction loss and $\mathcal{L}_{adv}$ is the adversarial loss. The reconstruction loss is given by:
\begin{equation}
    \mathcal{L}_{rec}(\mathbf{x})=\Vert\mathbf{w}\circ(\mathbf{x}-F((\mathbf{1}-\mathbf{w})\circ x)\Vert_2^2,
\end{equation}
where $\mathbf{w}$ is a binary mask, $\mathbf{x}$ is an example image and CE produces an output $F(\mathbf{x})$. The adversarial loss is based on Generative Adversarial Networks (GAN). GAN learns both a generative model $G_i$ from noise distribution $\mathcal{Z}$ to data distribution $\mathcal{X}$ and a discriminative model $D_i$ by the following objective:
\begin{equation}
\mathcal{L}_{a_i}=\min_{G_i}\max_{D_i} \mathbb{E}_{\mathbf{x}\in\mathcal{X}}[\log(D_i(\mathbf{x}))] + \mathbb{E}_{\mathbf{z}\in\mathcal{Z}}[\log(1-D_i(G_i(\mathbf{z})))].
\end{equation}
For CE, the adversarial loss is modified to
\begin{equation}
    \mathcal{L}_{adv} = \max_D\mathbb{E}_{\mathbf{x}\in\mathcal{X}}[\log(D(\mathbf{x}))+\log(1-D(F((\mathbf{1}-\mathbf{w})\circ \mathbf{x})))].
\end{equation}
Generative face completion\cite{Li2017} uses two discriminators instead with the following objective
\begin{equation}
\mathcal{L}=\mathcal{L}_{p}+ \lambda_{1} \mathcal{L}_{a_1} + \lambda_{2} \mathcal{L}_{a_2},
\end{equation}
where $\mathcal{L}_p$ is a parsing loss of pixel-wise softmax between the estimated UV texture $I_{i,j}$ and the ground truth texture $I_{i,j}^{*}$ of width $W$ and height $H$
\begin{equation}
L_{p}=\frac{1}{W\times H}\sum_{i=1}^{W}\sum_{j=1}^{H}\left | I_{i,j}-I_{i,j}^{*} \right |.
\end{equation}
Patch synthesis\cite{Yang2017} optimises a loss function of three terms: the holistic content term, the local texture term and the TV-loss term. The content constraint penalises the $l_2$ difference between the optimisation result and the previous content prediction
\begin{equation}
    l_c=\Vert \mathbf{w}\circ(\mathbf{x} - \mathbf{x}_i)\Vert_2^2,
\end{equation}
where $\mathbf{x}_i$ if the optimisation result at a coarser scale. The texture constraint penalises the texture appearance across the hole,
\begin{equation}
    l_t=\frac{1}{|\mathbf{w}^\phi|}\sum_{i\in\mathbf{w}^\phi}\Vert P_i\circ\phi(\mathbf{x})-P_{nn(i)}\circ\phi(\mathbf{x})\Vert_2^2,
\end{equation}
where $\mathbf{w}^\phi$ is the corresponding mask in the VGG-19 feature map $\phi(\mathbf{x})$, $|\mathbf{w}^\phi|$ is the number of patches sampled in $\mathbf{w}^\phi$, $P_i$ is the local neural patch at location $i$, and $nn(i)$ is the nearest neighbor of $i$. The TV loss encourages smoothness:
\begin{equation}
    l_{TV}=\sum_{i,j}((\mathbf{x}_{i,j+1}-\mathbf{x}_{i,j})^2+(\mathbf{x}_{i+1,j}-\mathbf{x}_{i,j})^2).
\end{equation}

\section{Robust Principal Component Analysis Using Side Information}

In this section, we propose models of RPCA using side information. In particular, we incorporate side information into PCP by using the trace distance of the difference between the low-rank component and the noisy estimate, which can be seen as a generalisation of compressed sensing with prior information where $l_1$ norm has been used to minimise the distance between the target signal and side information \cite{Mota17}.

\subsection{The PCPSM and PCPSFM models}

Assuming that a noisy estimate of the low-rank component of the data $\mathbf{S}\in\mathbb{R}^{n_1\times n_2}$ is available, we propose the following model of PCP using side information with missing values (PCPSM):
\begin{equation}
    \begin{split}
        &\min_{\mathbf{L},\mathbf{E}}\quad\Vert \mathbf{L}\Vert_*+\alpha\Vert \mathbf{L}-\mathbf{S}\Vert_*+\lambda\Vert\mathbf{W}\circ\mathbf{E}\Vert_1\\
        &\subject\ \quad\mathbf{X}=\mathbf{L}+\mathbf{E},
    \end{split}
\end{equation}
where $\alpha>0,\lambda>0$ are parameters that weigh the effects of side information and noise sparsity. 

The proposed PCPSM can be revamped to generalise the previous attempt of PCPF by the following objective of PCP using side information with features and missing values (PCPSFM):
\begin{equation}\label{eq:PCPSFM}
    \begin{split}
        &\min_{\mathbf{H},\mathbf{E}}\quad\Vert\mathbf{H}\Vert_*+\alpha\Vert\mathbf{H}-\mathbf{D}\Vert_*+\lambda\Vert\mathbf{W}\circ\mathbf{E}\Vert_1\\
        &\subject\ \quad\mathbf{X}=\mathbf{U}\mathbf{H}\mathbf{V}^T+\mathbf{E},\quad \mathbf{D}=\mathbf{U}^T\mathbf{S}\mathbf{V},\\
    \end{split}
\end{equation}
where $\mathbf{H}\in\mathbb{R}^{d_1\times d_2},\mathbf{D}\in\mathbb{R}^{d_1\times d_2}$ are bilinear mappings for the recovered low-rank matrix $\mathbf{L}$ and side information $\mathbf{S}$ respectively. Note that the low-rank matrix $\mathbf{L}$ is recovered from the optimal solution ($\mathbf{H}^*,\mathbf{E}^*$) to objective (\ref{eq:PCPSFM}) via $\mathbf{L}=\mathbf{U}\mathbf{H}^*\mathbf{V}^T$. If side information $\mathbf{S}$ is not available, PCPSFM reduces to PCPF with missing values by setting $\alpha$ to zero. If the features $\mathbf{U},\mathbf{V}$ are not present either, PCP with missing values can be restored by fixing both of them at identity. However, when only the side information $\mathbf{S}$ is accessible, objective (\ref{eq:PCPSFM}) is transformed back into PCPSM.

\subsection{The algorithm}

If we substitute $\mathbf{B}$ for $\mathbf{H}-\mathbf{D}$ and orthogonalise $\mathbf{U}$ and $\mathbf{V}$, the optimisation problem (\ref{eq:PCPSFM}) is identical to the following convex but non-smooth problem:
\begin{equation}\label{eq:PCPSFM2}
    \begin{split}
        &\min_{\mathbf{H},\mathbf{E}}\quad\Vert\mathbf{H}\Vert_*+\alpha\Vert\mathbf{B}\Vert_*+\lambda\Vert\mathbf{W}\circ\mathbf{E}\Vert_1\\
        &\subject\ \quad\mathbf{X}=\mathbf{U}\mathbf{H}\mathbf{V}^T+\mathbf{E},\quad\mathbf{B}=\mathbf{H}-\mathbf{U}^T\mathbf{S}\mathbf{V},
    \end{split}
\end{equation}
which is amenable to the multi-block alternating direction method of multipliers (ADMM).

The corresponding augmented Lagrangian of (\ref{eq:PCPSFM2}) is:
\begin{equation}\label{eq:PCPSFMl}
\begin{split}
&l(\mathbf{H},\mathbf{B},\mathbf{E},\mathbf{Z},\mathbf{N})=\Vert\mathbf{H}\Vert_*+\alpha\Vert\mathbf{B}\Vert_*+\lambda\Vert\mathbf{W}\circ\mathbf{E}\Vert_1\\&+\langle\mathbf{Z},\mathbf{X}-\mathbf{E}-\mathbf{U}\mathbf{H}\mathbf{V}^T\rangle +\frac{\mu}{2}\Vert\mathbf{X}-\mathbf{E}-\mathbf{U}\mathbf{H}\mathbf{V}^T\Vert_F^2\\&+\langle\mathbf{N},\mathbf{H}-\mathbf{B}-\mathbf{U}^T\mathbf{S}\mathbf{V}\rangle+\frac{\mu}{2}\Vert\mathbf{H}-\mathbf{B}-\mathbf{U}^T\mathbf{S}\mathbf{V}\Vert_F^2,
\end{split}
\end{equation}
where $\mathbf{Z}\in\mathbb{R}^{n_1\times n_2}$ and $\mathbf{N}\in\mathbb{R}^{d_1\times d_2}$ are Lagrange multipliers and $\mu$ is the learning rate.

The ADMM operates by carrying out repeated cycles of updates till convergence. During each cycle, $\mathbf{H},\mathbf{B},\mathbf{E}$ are updated serially by minimising (\ref{eq:PCPSFMl}) with other variables fixed. Afterwards, Lagrange multipliers $\mathbf{Z},\mathbf{N}$ are updated at the end of each iteration. Direct solutions to the single variable minimisation subproblems rely on  the shrinkage and the singular value thresholding operators \cite{Candes11}. Let $\mathcal{S}_\tau(a)\equiv\sgn(a)\max(|a|-\tau,0)$ serve as the shrinkage operator, which naturally extends to matrices,  $\mathcal{S}_\tau(\mathbf{A})$, by applying it to matrix $\mathbf{A}$ element-wise. Similarly, let $\mathcal{D}_\tau(\mathbf{A})\equiv \mathbf{M}\mathcal{S}_\tau(\mathbf{\Sigma})\mathbf{Y}^T$ be the singular value thresholding operator on real matrix $\mathbf{A}$, with $\mathbf{A}=\mathbf{M}\mathbf{\Sigma}\mathbf{Y}^T$ being the singular value decomposition (SVD) of $\mathbf{A}$.

Minimising (\ref{eq:PCPSFMl}) w.r.t. $\mathbf{H}$ at fixed $\mathbf{B},\mathbf{E},\mathbf{Z},\mathbf{N}$ is equivalent to the following:
\begin{equation}
    \arg\min_{\mathbf{H}}\ \Vert\mathbf{H}\Vert_*+\mu\Vert\mathbf{P}-\mathbf{U}\mathbf{H}\mathbf{V}^T\Vert_F^2,
\end{equation}
where $\mathbf{P} = \frac{1}{2}(\mathbf{X}-\mathbf{E}+\frac{1}{\mu}\mathbf{Z}+\mathbf{U}(\mathbf{B}+\mathbf{U}^T\mathbf{S}\mathbf{V}-\frac{1}{\mu}\mathbf{N})\mathbf{V}^T)$. Its solution is shown to be $\mathbf{U}^T\mathcal{D}_{\frac{1}{2\mu}}(\mathbf{P})\mathbf{V}$. Furthermore, for $\mathbf{B}$,
\begin{equation}
    \arg\min_{\mathbf{B}}\ l=\arg\min_{\mathbf{B}}\ \alpha\Vert\mathbf{B}\Vert_*+\frac{\mu}{2}\Vert\mathbf{Q}-\mathbf{B}\Vert_F^2,
\end{equation}
where $\mathbf{Q}=\mathbf{H}-\mathbf{U}^T\mathbf{S}\mathbf{V}+\frac{1}{\mu}\mathbf{N}$, whose update rule is $\mathcal{D}_{\frac{\alpha}{\mu}}(\mathbf{Q})$, and for $\mathbf{E}$,
\begin{equation}
    \arg\min_{\mathbf{E}}\ l=\arg\min_{\mathbf{E}}\ \lambda\Vert\mathbf{W}\circ\mathbf{E}\Vert_1+\frac{\mu}{2}\Vert\mathbf{R}-\mathbf{E}\Vert_F^2,
\end{equation}
where $\mathbf{R}=\mathbf{X}-\mathbf{U}\mathbf{H}\mathbf{V}^T+\frac{1}{\mu}\mathbf{Z}$ with a closed-form solution $\mathcal{S}_{\lambda\mu^{-1}}(\mathbf{R})\circ\mathbf{W}+\mathbf{R}\circ(\mathbf{1}-\mathbf{W})$. Finally, Lagrange multipliers are updated as usual:
\begin{equation}
    \mathbf{Z}= \mathbf{Z}+\mu(\mathbf{X}-\mathbf{E}-\mathbf{U}\mathbf{H}\mathbf{V}^T),
\end{equation}
\begin{equation}
    \mathbf{N} =\mathbf{N} +\mu (\mathbf{H}-\mathbf{B}-\mathbf{U}^T\mathbf{S}\mathbf{V}).
\end{equation}
The overall algorithm is summarised in Algorithm \ref{alg:PCPSFMa}.
\begin{algorithm}[tb]
   \caption{ADMM solver for PCPSFM}
   \label{alg:PCPSFMa}
\begin{algorithmic}[1]
   \REQUIRE Observation $\mathbf{X}$, mask $\mathbf{W}$, side information $\mathbf{S}$, features $\mathbf{U},\mathbf{V}$, parameters $\alpha,\lambda>0$, scaling ratio $\beta>1$.
   \STATE {\bfseries Initialize:} $\mathbf{Z}=0$, $\mathbf{N}=\mathbf{B}=\mathbf{H}=0$, $\beta=\frac{1}{\Vert\mathbf{X}\Vert_2}$.
   \WHILE{not converged}
   \STATE $\mathbf{E}=\mathcal{S}_{\lambda\mu^{-1}}(\mathbf{X}-\mathbf{U}\mathbf{H}\mathbf{V}^T+\frac{1}{\mu}\mathbf{Z})\circ\mathbf{W}+(\mathbf{X}-\mathbf{U}\mathbf{H}\mathbf{V}^T+\frac{1}{\mu}\mathbf{Z})\circ(\mathbf{1}-\mathbf{W})$
   \STATE $\mathbf{H}=\mathbf{U}^T\mathcal{D}_{\frac{1}{2\mu}}(\frac{1}{2}(\mathbf{X}-\mathbf{E}+\frac{1}{\mu}\mathbf{Z}+\mathbf{U}(\mathbf{B}+\mathbf{U}^T\mathbf{S}\mathbf{V}-\frac{1}{\mu}\mathbf{N})\mathbf{V}^T))\mathbf{V}$
   \STATE $\mathbf{B}=\mathcal{D}_{\alpha\mu^{-1}}(\mathbf{H}-\mathbf{U}^T\mathbf{S}\mathbf{V}+\frac{1}{\mu}\mathbf{N})$
   \STATE $\mathbf{Z} = \mathbf{Z}+\mu(\mathbf{X}-\mathbf{E}-\mathbf{U}\mathbf{H}\mathbf{V}^T)$
   \STATE $\mathbf{N} =\mathbf{N} +\mu (\mathbf{H}-\mathbf{B}-\mathbf{U}^T\mathbf{S}\mathbf{V})$
   \STATE $\mu = \mu\times\beta$
   \ENDWHILE
   \ENSURE $\mathbf{L}=\mathbf{U}\mathbf{H}\mathbf{V}^T$, $\mathbf{E}$
\end{algorithmic}
\end{algorithm}

\subsection{Complexity and convergence}

Orthogonalisation of the features $\mathbf{U},\mathbf{V}$ via the Gram-Schmidt process has an operation count of $O(n_1d_1^2)$ and $O(n_2d_2^2)$ respectively. The $\mathbf{H}$ update in Step $4$ is the most costly step of each iteration in Algorithm \ref{alg:PCPSFMa}. Specifically, the SVD required in the singular value thresholding action dominates with $O(\min(n_1n_2^2,n_1^2n_2))$ complexity.

A direct extension of the ADMM has been applied to our 3-block separable convex objective. Its global convergence is proved in \textbf{Theorem 1}. We have also used
the fast continuation technique already applied to the matrix
completion problem\cite{Toh2010} to increase $\mu$ incrementally
for accelerated superlinear performance\cite{Rockafellar76}. The cold start initialisation strategies for variables $\mathbf{H},\mathbf{B}$ and Lagrange multipliers $\mathbf{Z},\mathbf{N}$ are described in \cite{Boyd11}. Besides, we have scheduled $\mathbf{E}$ to be updated first and taken the initial learning
rate $\mu$ as suggested in\cite{lin09}. As for stopping criteria, we have employed the Karush-Kuhn-Tucker (KKT) feasibility conditions. Namely, within a maximum number of $1000$ iterations, when the maximum of $\Vert\mathbf{X}-\mathbf{E}_k-\mathbf{U}\mathbf{H}_k\mathbf{V}^T\Vert_F/\Vert\mathbf{X}\Vert_F$ and $\Vert\mathbf{H}_k -\mathbf{B}_k - \mathbf{U}^T\mathbf{S}\mathbf{V}\Vert_F/\Vert\mathbf{X}\Vert_F$ dwindles from a pre-defined threshold $\epsilon$, the algorithm is terminated, where $k$ signifies values at the $k$\textsuperscript{th} iteration.

\begin{innercustomthm}
    Let the iterative squence $\{(\mathbf{E}^k,\mathbf{H}^k,\mathbf{B}^k,\mathbf{Z}^k,\mathbf{N}^k)\}$ be generated by the direct extension of ADMM, Algorithm \ref{alg:PCPSFMa}, then the sequence $\{(\mathbf{E}^k,\mathbf{H}^k,\mathbf{B}^k,\mathbf{Z}^k,\mathbf{N}^k)$ converges to a solution of the Karush-Kuhn-Tucher (KKT) system in the fully observed case.
\end{innercustomthm}
\begin{proof}
    We first show that function $\theta_3(x_3)=\Vert E\Vert_1$ is sub-strong monotonic. From\cite{Candes11}, we know that $(x_1^*,x_2^*,x_3^*,\lambda^*)=(\mathbf{H}_0,\mathbf{E}_0,\mathbf{B}_0,\mathbf{Z}_0)$ is a KKT point, where $\mathbf{H}_0=\mathbf{U}^T\mathbf{L}_0\mathbf{V}$, $\mathbf{B}_0 = \mathbf{H}_0-\mathbf{U}^T\mathbf{S}\mathbf{V}$, $\mathbf{Z}_{0ij}=\lambda[\sgn(\mathbf{E}_0)]_{ij}$, if $(i,j)\in\Omega$ and $|\mathbf{Z}_{0ij}|<\lambda$, otherwise. Since $\theta_3(x_3)$ is convex, by definition, we have
    \begin{equation}
    \label{eq:def}
        \theta_3(x_3^*)\ge\theta_3(x_3)+\langle y_3,x_3^*-x_3\rangle,\quad\forall x_3\text{ and }\forall y_3\in	\partial\theta_3(x_3).
    \end{equation}
    Since $A_3$ is identity in (\ref{eq:PCPSFM2}), we have
    \begin{equation}
    \begin{split}
        &\theta_3(x_3)-\theta_3(x_3^*)+\langle A_3^T\lambda^*,x_3^*-x_3\rangle\\=&\lambda\Vert\mathbf{E}\Vert_1-\lambda\Vert\mathbf{E}_0\Vert_1+\langle\mathbf{Z}_0,\mathbf{E}_0\rangle - \langle\mathbf{Z}_0,\mathbf{E}\rangle,\\=&\lambda\Vert\mathbf{E}\Vert_1-\langle\mathbf{Z}_0,\mathbf{E}\rangle\\\ge&0,
    \end{split}
    \end{equation}
    where the third line follows from $\mathbf{Z}_{0ij}=\lambda[\sgn(\mathbf{E}_0)]_{ij}$ when $(i,j)\in\Omega$ and $\mathbf{E}_{0ij}=0$ when $(i,j)\notin\Omega$, and the fourth line follows from $|\mathbf{Z}_{0ij}|\le\lambda$, $|\mathbf{Z}_{0ij}\mathbf{E}_{ij}|\le|\mathbf{Z}_{0ij}||\mathbf{E}_{ij}|$ and $\Vert\mathbf{E}\Vert_1=\sum_{i,j}|\mathbf{E}_{ij}|$. But $\mathbf{E}$ is bounded, so there always exists $\mu>0$ such that
    \begin{equation}
        \lambda\Vert\mathbf{E}\Vert_1-\langle\mathbf{Z}_0,\mathbf{E}\rangle\ge\mu\Vert\mathbf{E}-\mathbf{E}_0\Vert_F^2.
    \end{equation}
    Thus, overall we have
    \begin{equation}
        \theta_3(x_3)\ge\theta_3(x_3^*)+\langle A_3^T\lambda^*,x_3-x_3^*\rangle+\mu\Vert\mathbf{E}-\mathbf{E}_0\Vert_F^2.
    \end{equation}
    Combining with (\ref{eq:def}), we arrive at
    \begin{equation}
        \langle y_3-A_3^T\lambda^*,x_3-x_3^*\rangle\ge\mu|x_3-x_3^*|^2,\ \forall x_3\text{ and }\forall y_3\in	\partial\theta_3(x_3),
    \end{equation}
    which shows that $\Vert E\Vert_1$ satisfies the sub-strong monotonicity assumption.

Additionally, $\Vert\mathbf{H}\Vert_*,\Vert\mathbf{B}\Vert_*$ are close and proper convex and $A$'s have full column rank. We thus deduce that the direct extension of ADMM, Algorithm \ref{alg:PCPSFMa}, applied to objective (\ref{eq:PCPSFM2}) is convergent according to\cite{Sun16}.\
\end{proof}

\section{Experimental results}

\subsection{Parameter calibration}
\begin{figure}[b]
\begin{center}
   \includegraphics[width=.9\linewidth]{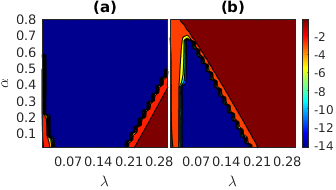}
\end{center}
   \caption{Log-scale relative error ($\log\frac{\Vert\mathbf{L}-\mathbf{L}_0\Vert_F}{\Vert\mathbf{L}_0\Vert_F}$) of PCPSM \textbf{(a)} when side information is perfect ($\mathbf{S}=\mathbf{L}_0$) and \textbf{(b)} when side information is the observation ($\mathbf{S}=\mathbf{M}$).}
    \label{fig:param}
\end{figure}

In this section, we illustrate the enhancement made by
side information through both numerical simulations and
real-world applications. First, we explain how parameters
used in our implementation are tuned. Second, we compare
the recoverability of our proposed algorithms with
state-of-the-art methods for incorporating features or dictionary,
viz. PCPF [17] and LRR [11] on synthetic data as well
as the baseline PCP [9] when there are no features available. Last, we show how powerful side information can be
for the task of UV completion in post-invariant face recognition, where both features and side information are derived from generative adversarial networks. 

For LRR, clean subspace X is used as in\cite{Sagonas14} instead of the observation $\mathbf{X}$ itself as the dictionary. PCP is solved via the inexact ALM\cite{lin09} and the heuristics for predicting the dimension of principal singular space is not adopted here
due to its lack of validity on uncharted real data\cite{Hintermüller15}. We also include Partial Sum of Singular Values (PSSV)\cite{Oh16} in
our comparison for its stated advantage in view of the limited number of images available. The
stopping criteria for PCPF, LRR, PCP and PSSV are all set
to the same KKT optimality conditions for reasons of consistency.

\begin{figure*}[t]
\begin{center}
\includegraphics[width=1\linewidth]{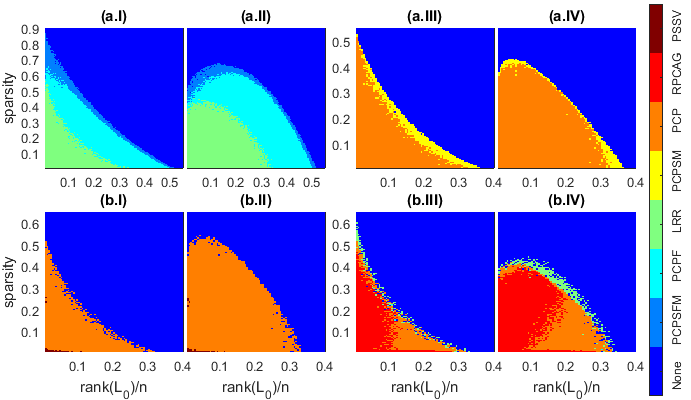}
\end{center}
   \caption{Domains of recovery by various algorithms in the fully observed case: \textbf{(I,III)} for random signs and \textbf{(II,IV)} for coherent signs.}
    \label{fig:full}
\end{figure*}
\begin{figure*}[b!]
\begin{center}
\includegraphics[width=1\linewidth]{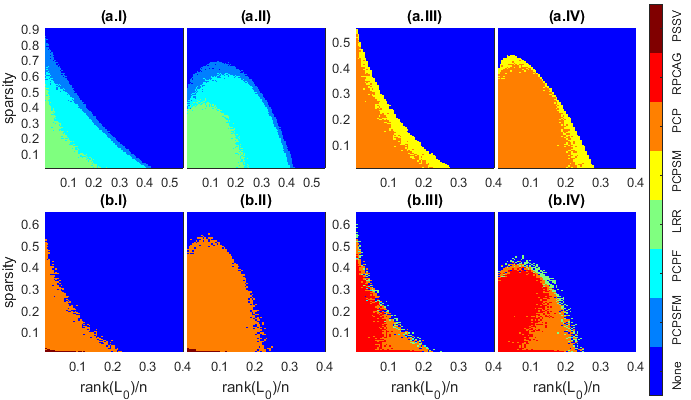}
\end{center}
   \caption{Domains of recovery by various algorithms in the partially observed case: \textbf{(I,III)} for random signs and \textbf{(II,IV)} for coherent signs.}
   \label{fig:partial}
\end{figure*}

\begin{figure*}[b]
\begin{center}
   \includegraphics[width=1\linewidth]{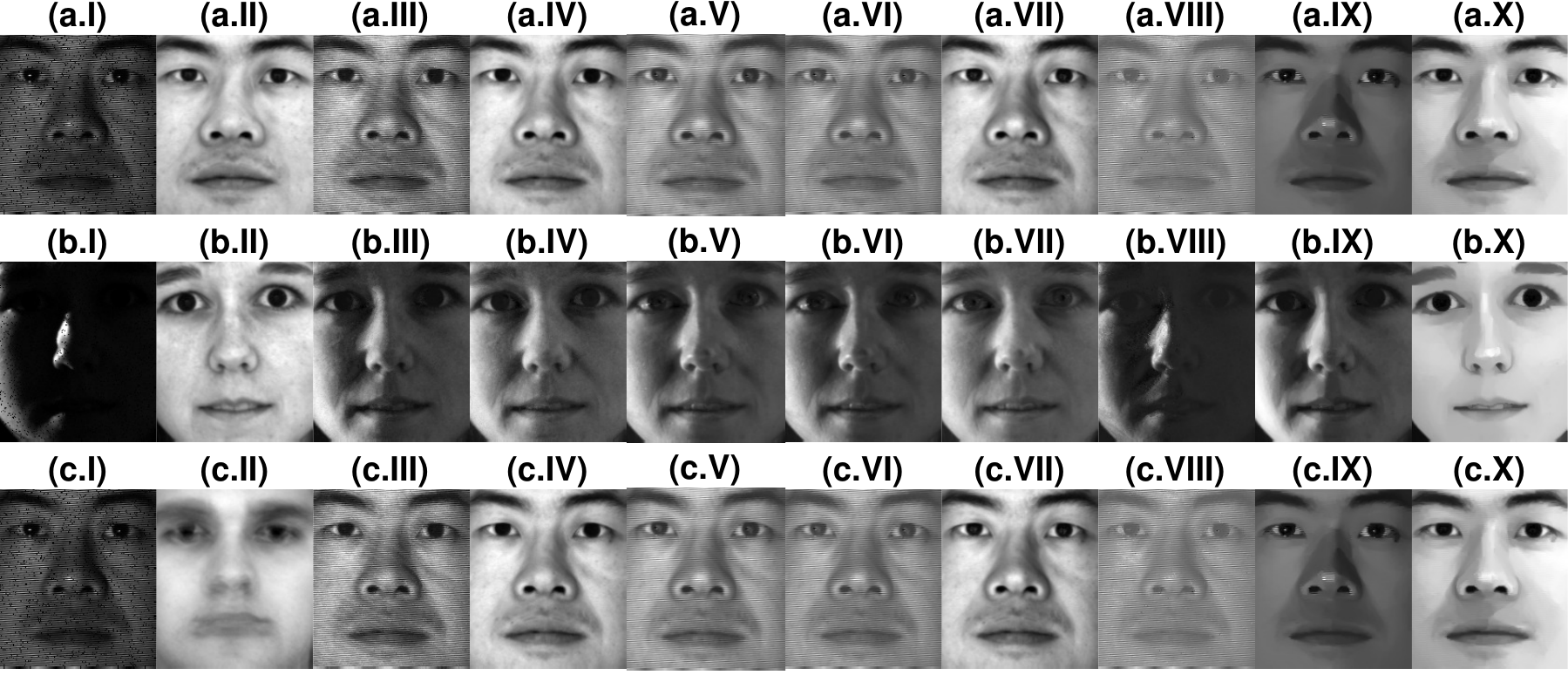}
\end{center}
   \caption{Comparison of face denoising ability: \textbf{(I)} Observation; \textbf{(II)} side information; \textbf{(III)} PCP; \textbf{(IV)} PCPSM; \textbf{(V)} LRR; \textbf{(VI)} PCPF; \textbf{(VII)} PCPFSM; \textbf{(VIII)} PSSV; \textbf{(IX)} RPCAG; and \textbf{(X)} FRPCAG.}
    \label{fig:Yale}
\end{figure*}

In order to tune the algorithmic parameters, we first conduct a benchmark experiment as follows: a low-rank matrix $\mathbf{L}_0$ is generated from $\mathbf{L}_0=\mathbf{J}\mathbf{K}^T$, where $\mathbf{J},\mathbf{K}\in\mathbb{R}^{200\times10}$
have entries from a $\mathcal{N}(0, 0.005)$ distribution; a $200 \times 200$ sparse matrix $\mathbf{E}_0$ is generated by randomly setting $38000$ entries to zero with others taking values of $\pm1$ with equal probability; side information $\mathbf{S}$ is assumed perfect, that is, $\mathbf{S} = \mathbf{L}_0$; $\mathbf{U}$ is set as the left-singular vectors of $\mathbf{L}_0$; and $\mathbf{V}$ is set as the right-singular vectors of $\mathbf{L}_0$; all entries are observed. It has been found
that a scaling ratio $\beta=1.1$, a tolerance threshold $\epsilon= 10^{-7}$ and a maximum step size $\mu=10^{7}$ to avoid ill-conditioning can bring all models except PSSV to convergence with a recovered $\mathbf{L}$ of rank $10$, a recovered $\mathbf{E}$ of sparsity $5\%$ and an
accuracy $\Vert\mathbf{L}-\mathbf{L}_0\Vert_F/\Vert\mathbf{L}_0\Vert_F$ on the order of $10^{-6}$. Still,
these apply to PSSV as is done similarly in\cite{Oh16}.

Although theoretical determination of $\alpha$ and $\lambda$ is beyond the scope of this paper, we nevertheless provide empirical guidance based on extensive experiments. A parameter weep in the $\alpha-\lambda$ space for perfect side information is
shown in Figure \ref{fig:param}(a) and for observation as side information
in Figure \ref{fig:param}(b) to impart a lower bound and a upper bound
respectively. It can be easily seen that $\lambda = 1/\sqrt{200}$ (or $\lambda = 1/\sqrt{\max(n_1,n_2)}$ for a general matrix of dimension
$n_1\times n_2$) from Robust PCA works well in both cases. Conversely,
$\alpha$ depends on the quality of the side information.
When the side information is accurate, a large $\alpha$ should be
selected to capitalise upon the side information as much as
possible, whereas when the side information is improper, a
small $\alpha$ should be picked to sidestep the dissonance caused
by the side information. Here, we have discovered that a κ
value of 0.2 works best with synthetic data and a value of
0.5 is suited for public video sequences, both of which will
be used in all experiments in subsequent sections together
with other aforementioned parameter settings. It is worth
emphasising again that prior knowledge of the structural information
about the data yields more appropriate values for $\alpha$ and $\lambda$.

\subsection{Phase transition on synthetic datasets}

We now focus on the recoverability problem, i.e. recovering matrices of varying ranks from errors of varying sparsity. True low-rank matrices are created via $\mathbf{L}_0=\mathbf{J}\mathbf{K}^T$, where $200\times r$ matrices $\mathbf{J},\mathbf{K}$ have independent elements drawn randomly from a Gaussian distribution of mean $0$ and variance $5\cdot10^{-3}$ so $r$ is the rank of $\mathbf{L}_0$. Next, we generate $200\times200$ error matrices $\mathbf{E}_0$, which possess $\rho_s\cdot200^2$ non-zero elements located randomly within the matrix. We consider two types of entries for $\mathbf{E}_0$: Bernoulli $\pm1$ and $\mathcal{P}_\Omega(\sgn(\mathbf{L}_0))$, where $\mathcal{P}$ is the projection operator. $\mathbf{X}=\mathbf{L}_0+\mathbf{E}_0$ thus becomes the simulated observation. For each $(r,\rho_s)$ pair, three observations are constructed. The recovery is successful if for all these three problems, 
\begin{equation}
    \frac{\Vert\mathbf{L}-\mathbf{L}_0\Vert_F}{\Vert\mathbf{L}_0\Vert_F}<10^{-3}
\end{equation}
from the recovered $\mathbf{L}$. In addition, let $\mathbf{L}_0=\mathbf{M}\mathbf{\Sigma}\mathbf{Y}^T$ be the SVD of $\mathbf{L}_0$. Feature $\mathbf{U}$ is formed by randomly interweaving column vectors of $\mathbf{M}$ with $d$ arbitrary orthonormal bases for the null space of $\mathbf{M}^T$, while permuting the expanded columns of $\mathbf{Y}$ with $d$ random orthonormal bases for the kernel of $\mathbf{Y}^T$ forms feature $\mathbf{V}$. Hence, the feasibility conditions are fulfilled: $\mathbb{C}(\mathbf{U})\supseteq\mathbb{C}(\mathbf{L}_0)$, $\mathbb{C}(\mathbf{V})\supseteq\mathbb{C}(\mathbf{L}_0^T)$, where $\mathbb{C}$ is the column space operator.

For each trial, we construct the side information by directly adding small Gaussian noise to each element of $\mathbf{L}_0$: $l_{ij}\rightarrow l_{ij}+\mathcal{N}(0,2.5r\cdot10^{-9})$, $i,j=1,2,\cdots,200$. As a result, the standard deviation of the error in each element is $1\%$ of that among the elements themselves. On average, the Frobenius percent error, $\Vert\mathbf{S}-\mathbf{L}_0\Vert_F/\Vert\mathbf{L}_0\Vert_F$, is $1\%$. Such side information is genuine in regard to the fact that classical PCA with accurate rank is not able to eliminate the noise \cite{Shabalin13}. We set $d$ to 10 throughout.

\textbf{Full observation} Figures \ref{fig:full}(a.I) and (a.II) plot results from PCPF, LRR and PCPSFM. On the other hand, the situation with no available features is investigated in Figures \ref{fig:full}(a.III) and \ref{fig:full}(a.IV) for PCP and PCPSM. The frontier of PCPF has been advanced by PCPSFM everywhere for both sign types. Especially at low ranks, errors with much higher density can be removed. Without features, PCPSM surpasses PCP by and large with significant expansion at small sparsity for both cases. Results from RPCAG and PSSV are worse than PCP with LRR marginally improving (see Figures \ref{fig:full}(b.I), (b.II), (b.III) and b(IV)).

\textbf{Partial observation} Figures \ref{fig:partial}(a.I) and (a.II) map out the results for PCPF, LRR and PCPSFM when $10\%$ of the elements are occluded and Figures \ref{fig:partial}(a.III) and (a.IV) for featureless PCP and PCPSM. In all cases, areas of recovery are reduced. However, there are now larger gaps between PCPF, PCPSFM and PCP, PCPSM. This marks the usefulness of side information particularly in the event of missing observations. We remark that in unrecoverable areas, PCPSM and PCPSFM still obtain much smaller values of $\Vert L-L_0\Vert_F$. FRPCAG fails to recover anything at all.

\subsection{Face denoising}

If a surface is convex Lambertian and the lighting is isotropic and distant, then the rendered model spans a 9-D linear subspace \cite{Basri03}. Nonetheless, facial images are only approximately so because facial harmonic planes have negative pixels and real lighting conditions entail unavoidable occlusion and albedo variations. It is thus more reasonable to decompose facial image formation as a low-rank component for face description and a sparse component for defects. In pursuit of this low-rank portrayal, we suggest that there can be further boost to the performance of facial characterisation by leveraging an image which faithfully represents the subject.

We consider images of a fixed pose under different illuminations from the extended Yale B database for testing. All 64 images were studied for each person. $32556\times64$ observation matrices were formed by vectorising each $168\times192$ image and the side information was chosen to be the average of all images, tiled to the same size as the observation matrix for each subject. In addition, $5\%$ of randomly selected pixels of each image were set as missing entries.

For LLR, PCPF and PCPSFM to run, we learn the feature dictionary following an approach by Vishal et al.~\cite{Patel12}. In a nutshell, the feature learning process can be treated as a sparse encoding problem. More specifically, we simultaneously seek a dictionary $\mathbf{D}\in\mathbb{R}^{n_1\times c}$ and a sparse representation $\mathbf{B}\in\mathbb{R}^{c\times n_2}$ such that:
\begin{equation}
\label{eq:KSVD}
    \min_{\mathbf{D},\mathbf{B}}\Vert\mathbf{M}-\mathbf{D}\mathbf{B}\Vert_F^2\quad\subject\quad\gamma_i\le t\ \for\  i=1\dots n_2,
\end{equation}
where $c$ is the number of atoms, $\gamma_i$'s count the number of non-zero elements in each sparsity code and $t$ is the sparsity constraint factor. This can be solved by the K-SVD algorithm\cite{Aharon06}. Here, feature $\mathbf{U}$ is the dictionary $\mathbf{D}$ and feature $\mathbf{V}$ corresponds to a similar solution using the transpose of the observation matrix as input. For implementation details, we set $c$ to $40$, $t$ to $40$ and used $10$ iterations for each subject.

As a visual illustration, two challenging cases are exhibited in Figure \ref{fig:Yale}. For subject $\#2$, it is clearly evident that PCPSM and PCPSFM outperform the best existing methods through the complete elimination of acquisition faults. More surprisingly, PCPSFM even manages to restore the flash in the pupils that is barely present in the side information. For subject $\#34$, PCPSM indubitably reconstructs a more vivid right eye than that from PCP which is only discernible. With that said, PCPSFM still prevails by uncovering more shadows, especially around the medial canthus of the right eye, and revealing a more distinct crease in the upper eyelid as well a more translucent iris. We further unmask the strength of PCPSM and PCPSFM by considering the stringent side information made of the average of 10 other subjects. Surprisingly, PCPSM and PCPSFM still manage to remove the noise, recovering an authentic image (Figures \ref{fig:Yale}(c.IV) and \ref{fig:Yale}(c.VII)). We also notice that PSSV, RPCAG, FRPCAG do not improve upon PCP as in simluation experiments. Thence, we will focus on comparisons with PCP, LRR, PCPF only.

\subsection{UV map completion}

We concern ourselves with the problem of completing the UV texture for each of a sequence of video frames. That is, we apply PCPSM and PCPSFM to a collection of incomplete textures lifted from the video. This parameter-free approach is advantageous to a statistical texture model such as 3D Morphable Model (3DMM)~\cite{Blanz03,Booth_2016_CVPR} by virtue of its difficulty in reconstructing unseen images captured 'in-the-wild' (any commercial camera in arbitrary conditions).

\subsubsection{Texture extraction}
Given a 2D image, we extract the UV texture by fitting 3DMM. Specifically, following \cite{booth20173d}, three parametric models are employed. There are a 3D shape model (\ref{equ:shape_model}), a texture model (\ref{equ:texture_model}) and a camera model (\ref{equ:camera_model}):
\begin{equation}
\mathcal{S}(\mathbf{p}) = \overline{\mathbf{s}} + \mathbf{U}_s \mathbf{p},
\label{equ:shape_model}
\end{equation}
\begin{equation}
\mathcal{T}(\boldsymbol{\lambda}) = \overline{\mathbf{t}} + \mathbf{U}_t \boldsymbol{\lambda},
\label{equ:texture_model}
\end{equation}
\begin{equation}
\mathcal{W}(\mathbf{p},\mathbf{c}) = \mathcal{P}(\mathcal{S}(\mathbf{p}),\mathbf{c}),
\label{equ:camera_model}
\end{equation}
where $\mathbf{p} \in \mathbb{R}^{n_s}, \boldsymbol{\lambda} \in \mathbb{R}^{n_t}$ and $\mathbf{c} \in \mathbb{R}^{n_c}$ are shape, texture and camera parameters to optimise; $\mathbf{U}_s \in \mathbb{R}^{3N \times n_s}$ and $\mathbf{U}_t \in \mathbb{R}^{3N \times n_t}$ are the shape and texture eigenbases respectively with $N$ the number of vertices in shape model; $\mathbf{\overline{s}} \in \mathbb{R}^{3N}$ and $\mathbf{\overline{t}} \in \mathbb{R}^{3N}$ are the means of shape and texture models correspondingly, learnt from 10000 face scans of different individuals\cite{Booth_2016_CVPR}; $\mathcal{P}(\mathbf{s},\mathbf{c}): \mathbb{R}^{3N} \rightarrow \mathbb{R}^{2N}$ is a perspective camera transformation function.

\begin{figure}[b!]
\begin{center}
\includegraphics[width=0.5\textwidth]{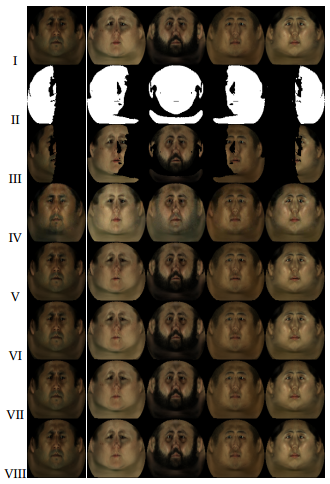}
\end{center}
\caption{(row I) original sequences; (row II) random masks; (row III) sample inputs;  (row IV) side information; (row V) PCP; (row VI) PCPSM; (row VII) LRR; (row VIII) PCPSFM.}
\label{fig:4DFAB}
\end{figure}

The complete cost function for 3DMM fitting is:
\begin{align}
\min_{\mathbf{p},\boldsymbol{\lambda},\mathbf{c}} 
&\| \mathbf{F}(\mathcal{W}(\mathbf{p},\mathbf{c})) - \mathcal{T}(\boldsymbol{\lambda}) \|^2 + 
\beta_l \| \mathcal{W}(\mathbf{p},\mathbf{c})) - \mathbf{s}_l \|^2\notag\\
& + \beta_s \| \mathbf{p} \|^2_{\mathbf{\Sigma}_{s}^{-1}} + \beta_t \| \boldsymbol{\lambda} \|^2_{\mathbf{\Sigma}_{t}^{-1}},
\label{equ:3dmm_fitting_final}
\end{align}
where $\mathbf{F}(\mathcal{W}(\mathbf{p},\mathbf{c}))$ denotes the operation of sampling the feature image onto the projected 2D locations. The second term is a landmark term with weighting $\beta_l$ used to accelerate in-the-wild 3DMM fitting, where 2D shape $\mathbf{s}_l$ is provided by\cite{bulat2017far}. The final two terms are regularisation terms to counter over-fitting, where $\mathbf{\Sigma}_{s}$ and $\mathbf{\Sigma}_{t}$ are diagonal matrices with the main diagonal being eigenvalues of the shape and texture models respectively. (\ref{equ:3dmm_fitting_final}) is solved by the Gauss-Newton optimisation framework (see\cite{booth20173d} for details).

\subsubsection{Quantitative evaluation}

We quantitatively evaluate the completed UV maps by our proposed methods on the 4DFAB dataset~\cite{cheng4dfab}. 4DFAB is the first 3D dynamic facial expression dataset designed for biometric applications, where 180 participants were invited to attend four sessions at different times. Hence, to complete UV maps for one session, we can leverage images from another session as the side information. For each of 5 randomly selected subjects, one dynamic sequence of 155 frames was randomly cut from the second session. After vectorisation, a $32556\times155$ observation matrix was formed. To produce UV masks of different poses, we rotate each face with different yaw and pitch angles. The yaw angle ranges from $-90^\circ$ to $90^\circ$ in steps of $6^\circ$, wheres the pitch angle is selected from $[ -10^\circ,-5^\circ,0^\circ,5^\circ,10^\circ]$. Therefore, for each subject, a set of 155 unique masks were generated. We also tiled one image of the same subject from the first session into a $32556\times155$ matrix as side information. $\mathbf{U}$ was provided by the left and right singular vectors of the original sequence while $\mathbf{V}$ was set to the identity.

From Figure \ref{fig:4DFAB}, we observe that (I) RPCA approaches can deal with cases when more than $50\%$ of the pixels are missing; (II) imperfect side information (shaved beard, removed earrings and different lights) still help with the recovery process. We record peak signal-to-noise ratios (PSNR) and structural similarity indices (SSIM) between the completed UV maps and the original map in Table \ref{table:psnr}. It is evident that with the assistance of side information much higher fidelity can be achieved. The use of imperfect side information nearly comes on a par with perfect features.

\begin{table}[t]
\begin{center}
\begin{tabular}{ |c|c|c|c|c|c|c| }
\hline
\multicolumn{2}{|c|}{Subject} & $\#1$ & $\#2$ & $\#3$ & $\#4$ & $\#5$ \\
\hline
PSNR&PCP&$35.99$&$26.75$&$32.65$&$31.33$&$29.10$\\
\cline{2-7}
(dB)&PCPSM&$\bm{39.56}$&$\bm{30.63}$&$\bm{34.66}$&$\bm{35.86}$&$\bm{32.80}$\\
\cline{2-7}
&LRR&$40.94$&$30.69$&$36.38$&$35.94$&$33.97$\\
\cline{2-7}
&PCPSFM&$\bm{41.48}$&$\bm{31.46}$&$\bm{37.29}$&$\bm{36.60}$&$\bm{34.80}$\\
\hline
SSIM&PCP&$0.973$&$0.922$&$0.962$&$0.956$&$0.949$\\
\cline{2-7}
&PCPSM&$\bm{0.98}$&$\bm{0.952}$&$\bm{0.969}$&$\bm{0.981}$&$\bm{0.973}$\\
\cline{2-7}
&LRR&$0.990$&$0.952$&$0.975$&$0.982$&$0.978$\\
\cline{2-7}
&PCPSFM&$\bm{0.991}$&$\bm{0.958}$&$\bm{0.980}$&$\bm{0.984}$&$\bm{0.981}$\\
\hline
\end{tabular}
\end{center}
\caption{Quantitative measures of UV completion by various algorithms on the 4DFAB dataset.}
\label{table:psnr}
\end{table}

\subsubsection{Generative adversarial networks}

More often than not, ground-truth $\mathbf{U}$, $\mathbf{V}$ are not accessible to us for in-the-wild videos. Learning methods such as (\ref{eq:KSVD}) must be leveraged to acquire $\mathbf{U}$ or $\mathbf{V}$. However, (\ref{eq:KSVD}) is not ideal: (I) it is not robust to errors; (II) it cannot handle missing values; (III) it requires exhaustive search of optimal parameters which vary from video to video; (IV) it only admits greedy solutions. On the other hand, we can use GAN to produce authentic pseudo ground-truth and then obtain $\mathbf{U}$ and $\mathbf{V}$ accordingly. Moreover, such completed sequence provides us good side information. For GAN, we employ the image-to-image conditional adversarial network\cite{isola2016image} (appropriately customised for UV map completion) to conduct UV completion. Details regarding the architecture and training of GAN can be found in the supplementary material.

\subsubsection{Qualitative demonstration}

To examine the ability of our proposed methods on in-the-wild images. We perform experiments on the 300VW dataset~\cite{300vw}. This dataset contains 114 in-the-wild videos that exhibit large variations in pose, expression, illumination, background, occlusion, and image quality. Each video shows exactly one person, and each frame is annotated with 68 facial landmarks. 
We performed 3DMM fitting on these videos and lifted one corresponding UV map for earch frame, where the visibility mask was produced by z-buffering based on the fitted mesh. The side information was generated by taking the average of the completed UVs from GAN. $\mathbf{U}$ and $\mathbf{V}$ were assigned to the singular vectors of the completed texture sequence from GAN.

\begin{figure*}[h!]
\begin{center}
\includegraphics[width=1\textwidth]{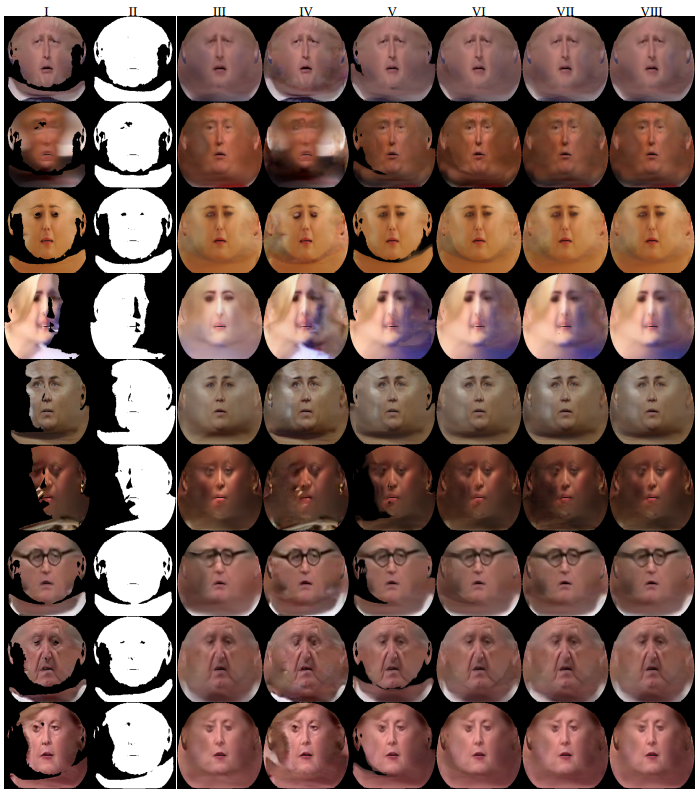}
\end{center}
\caption{(column I) sample frame; (column II) mask; (column III) side information; (column IV) GAN; (column V) PCP; (column VI) PCPSM; (column VII) LRR; (column VIII) PCPSFM.}
\label{fig:300vW}
\end{figure*}

\begin{figure*}[h!]
\begin{center}
\includegraphics[width=1\textwidth]{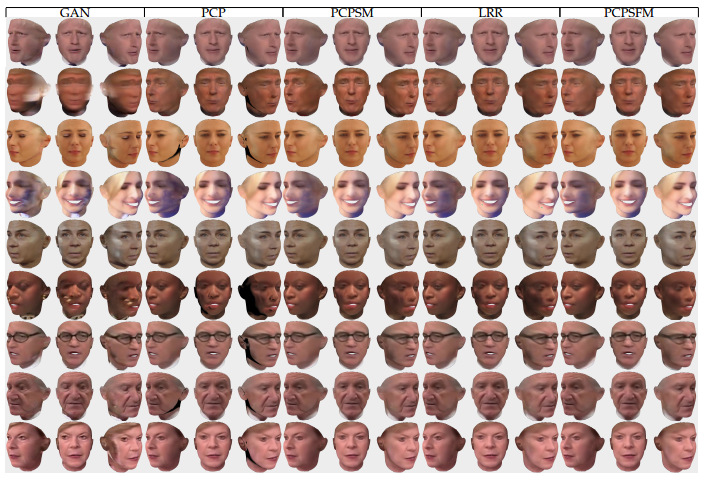}
\end{center}
\caption{2D face synthesis of three views ($-45^{\circ}$, $0^{\circ}$, $45^{\circ}$) from the completed UV maps by various methods.}
\label{fig:300vWShape}
\end{figure*}

We display results for one frame from each of 9 arbitrary videos in Figure \ref{fig:300vW}. As evident from the images, GAN alone has unavoidable drawbacks: (I) when the 3DMM fitting is not accurate, GAN is unable to correct such defects; (II) when the image itself contains errors, GAN is unable to remove them. On the other hand, PCP often fails to produce a complete UV. PCPSM always produces a completed UV texture, which is an improvement over PCP, but it makes boundaries undesirably visible. Visually, LRR and PCPSFM have the best performance, being able to produce good completed UVs for a large variety of poses, identities, lighting conditions and facial characteristics. This justifies the quality of subspaces and side information from GAN for use in the robust PCA framework. We also synthesise 2D faces of three different poses using the the completed UV maps in Figure \ref{fig:300vWShape}.

\subsection{Face recognition}

Face recognition is a crucial element of biometrics\cite{taigman2014deepface,sun2014deep,schroff2015facenet,parkhi2015deep,chen2017robust,ranjan2017l2,shekhar2017synthesis}. More specifically, to fully exploit the existing vast sources of in-the-wild data, one needs to recognise faces in a pose-invariant way. Modern approaches include pose-robust feature extraction, multi-view subspace learning, face synthesis, etc\cite{Ding16}. These often fall short of expectations either due to fundamental limitations or inability to fuse with other useful methods. For example, Generalized Multiview Analysis\cite{Sharma12} cannot take into account of pose normalisation\cite{Ding15} or deep neural network-based pose-robust feature extraction\cite{Kan14} and vice versa. It is thus fruitful to provide a framework where information from different perspectives can be fused together to deliver better prediction.

We quantitatively evaluate our proposed fusion methods by carrying out face recognition experiments on the completed UVs. The experiments are performed on three standard databases, \ie VGG, CFP and YTF. We perform both ablation as well as video-based face recognition experiments against known benchmarks.

\subsubsection{Datasets}

{\bf VGG} The VGG2 dataset\cite{cao2017vggface2} contains a training set of 8,631 identities (3,141,890 images) and a test set of 500 identities (169,396 images). VGG2 has large variations in pose, age, illumination, ethnicity and profession. To facilitate the evaluation of face matching across different poses, VGG2 also provides a face template list for each of some 368 subjects, which contains 2 front templates, 2 three-quarter templates and 2 profile templates. Each template includes 5 images. 

{\bf CFP} The CFP dataset\cite{sengupta2016frontal} consists of 500 subjects, each of which has 10 frontal and 4 profile images. As such, we define the evaluation protocol for frontal-frontal (FF) and frontal-profile (FP) face verification on 500 same-person pairs and 500 different-person pairs.

{\bf YTF} The YTF dataset\cite{wolf2011face} consists of $3,425$ videos of $1,595$ different people. The clip duration varies from $48$ frames to $6,070$ frames, with an average length of $181.3$ frames. We follow its \textit{unrestricted with labelled outside data} protocol and report the results on $5,000$ video pairs.

\subsubsection{Face Feature Embedding}
We use ResNet-27~\cite{zhang2016range,wen2016discriminative} for $512$-$d$ facial feature embedding with softmax loss (see Supplementary Material for more details). Figure~\ref{pic:setfacerecognition} illustrates the set-based face feature embedding used for face recognition and verification. After 3DMM fitting, we extract 3D face shapes and incomplete UV maps. Then, we utilise the proposed UV completion methods (GAN, PCP, PCPSM, LRR and PCPSFM) to derive compeleted UV maps. Frontal faces are synthesised from full UV maps and the 3D shapes through which $512$-$d$ features are obtained from the last fully connected layer of the feature embedding network. Finally, we calculate the feature centre as the feature description of this set of face images.

\begin{figure*}[b!]
\begin{center}
    \includegraphics[width=0.9\textwidth]{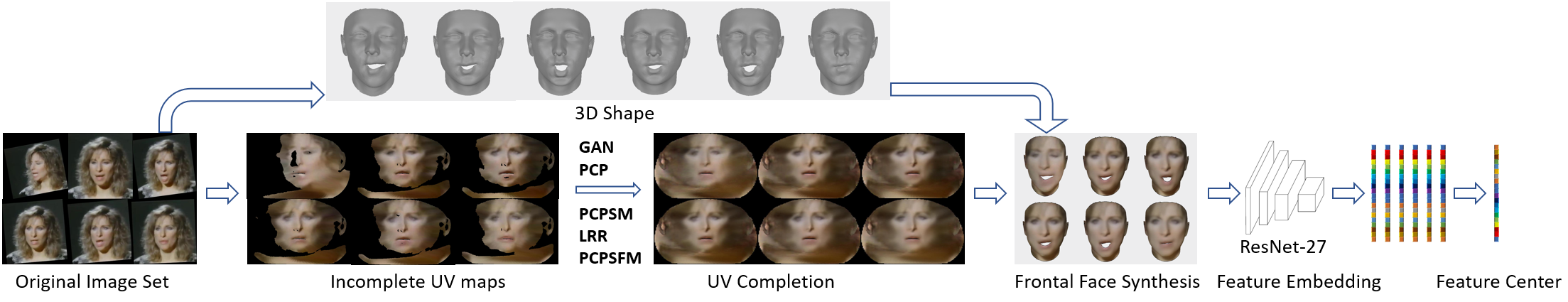}\\
\end{center}
\caption{The proposed pipeline for video-based face recognition. The 3DMM is fitted on the frames of the video and the in-completed UV maps are estimated.  The trained GAN is then used to provide a first estimate of the side information and the proposed methodology is applied to generate completed UV maps. The 3D model is again used to render the images in the frontal position. Deep neural network is used to extract features from all frames and the average of the features is used in order to represent the video.}
\label{pic:setfacerecognition}
\end{figure*}

\subsubsection{Ablation experiments}
We conduct face verification and identification experiments for the proposed methods on CFP and VGG datasets. Since the proposed methods are based on a sequence of images, we design our experiments for ablation study as follows: for experiments on VGG, we use the 368 identities from the pose template list, for each of which we divide the total 30 face images into two sets so that each set has 15 face images from frontal, three-quarter and profile templates; for experiments on CFP, we use all of the 500 identities, for each of which we divide the total 14 face images into two sets so that each set has 5 frontal face images and 2 profile face images. For face verification experiments, we construct $N$ positive set pairs and $N$ hard negative set pairs (nearest inter-class), where $N$ is 368 for the VGG dataset and 500 for the CFP dataset. For face identification experiments, one set of the face images from each identity is used as the gallery and the other set of the face images is taken to be the probe. So, for each probe, we predict its identity by searching for the nearest gallery in the feature space. Rank-1 accuracy is used here to assess the face identification performance.

\begin{figure*}[b]
\begin{center}
\includegraphics[width=0.9\linewidth]{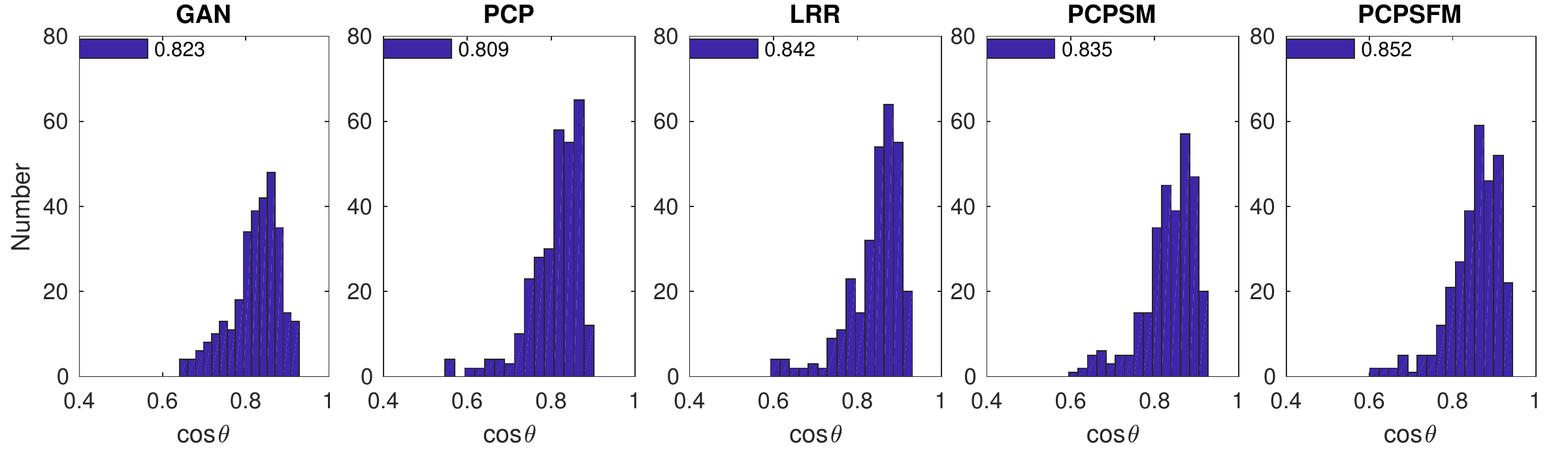}
\end{center}
   \caption{15-bin histograms of cos$\theta$ distributions produced from various methods. The legend in the figure indicates its mean.}
    \label{fig:costheta}
\end{figure*}

Detailed comparisons of the proposed methods against the baseline methods, PCP and LRR, are tabulated in Table \ref{tab:vgg2cfp}. For face verification experiments on the VGG dataset, we observe that feature subspace and side information from GAN (LRR and PCPSM) improve the vanilla PCP in terms of accuracy scores with a further boost in performance if both of them are considered together (PCPSFM). On the CFP dataset, the performance is slightly worse than that on VGG, but the improvement can still be made by exploiting features and side information obtained from GAN. A similar trend is seen for the face identification experiments. These results confirm the visual outcomes we see from the previous section. 

\begin{table}[h]
\small
\footnotesize
\begin{center}
\begin{tabular}{@{\hskip .5mm}l@{\hskip 1.5mm}|c@{\hskip 1.5mm}|c@{\hskip 1.5mm}|c@{\hskip 1.5mm}|c@{\hskip .5mm}}
\hline
Method            & \multicolumn{2}{c}{VGG} & \multicolumn{2}{c}{CFP} \\ \hline
Accuracy ($\%$)     & Ver. & Id. (Rank-1) & Ver. & Id. (Rank-1) \\ \hline
PCP    & 97.15 & 95.38 & 95.30 &  94.60\\
PCPSM  & 98.78 & 97.28 & 97.80 &  97.00\\
LRR    & 99.05 & 97.55 & 98.10 &  97.20\\
PCPSFM &$\bm{99.32}$ & $\bm{98.10}$ & $\bm{98.40}$ & $\bm{97.60}$\\
\hline
\end{tabular}
\end{center}
\caption{Face verification and identification results on VGG and CFP datasets.}
\label{tab:vgg2cfp}
\end{table}

\subsubsection{Video-based face recognition}

\begin{table}[t]
\begin{center}
\begin{tabular}{c|c|c}
\hline
Methods                        & Images          &YTF ($\%$)  \\
\hline\hline
DeepID~\cite{sun2014deep}      &  0.2M     &93.20\\
VGG Face~\cite{parkhi2015deep} & 2.6M      &$\bm{97.30}$\\
Deep Face~\cite{taigman2014deepface}&4M   &91.40\\
FaceNet~\cite{schroff2015facenet}&200M    &95.10\\
Center Loss~\cite{wen2016discriminative}&0.7M     &94.9 \\
Range Loss~\cite{zhang2016range}&1.5M    &93.70\\
Sphere Loss~\cite{liu2017sphereface} &0.5M &95.0\\
Marginal Loss~\cite{deng2017marginal} &4M  &95.98\\
\hline
GAN                    & 3.1M  & 95.68    \\
PCP                    & 3.1M  & 95.25    \\
PCPSM                  & 3.1M  & 95.94    \\
LRR                    & 3.1M  & 96.32    \\
PCPSFM                 & 3.1M  &$\bm{96.58}$\\
\hline
\hline
\end{tabular}
\end{center}
\caption{Verification performance of different methods on the YTF dataset}\label{table:YTF}
\end{table}

Face recognition on in-the-wild videos is a challenging task because of the rich pose variations. Videos from the YTF (YouTube Face) dataset also suffer from low resolution and serious compression artifacts which make the problem even worse. To conduct experiments on YTF, we follow the "restricted" protocol\cite{wolf2011face}, which forbids the use of subject identity labels during training. Therefore, we remove the 230 overlapping identities from VGG and re-train our GAN model for video-based face recognition experiments. For each method, individual frontal face image projection from the completed UV is produced for facial feature extraction. For each video, the final feature representation is a $512$-$d$ feature centre of all the frames.

We compare the face verification performance of the proposed methods with state-of-the-art approaches on the YTF dataset. The mean verification rates for best-performing deep learning methods are listed in Table \ref{table:YTF}. We see that our GAN model alone is among the best reported architectures and it outperforms the classical PCP. Nonetheless, their fusion (PCPSM, LRR and PCPSFM) is superior to either of them. More specifically, PCPSM improves PCP and GAN by $0.69\%$ and $0.26\%$ respectively. For LRR, the improvements are $1.07\%$ and $0.64\%$. PCPSFM has the most gain of $1.33\%$ over PCP and $0.9\%$ over LRR. Corresponding ROC curves are plotted in Figure~\ref{fig:ytf_rocs}, the proposed PCPSFM obviously improves the accuracy of video-based face recognition. We randomly select one sample video (Barbra Streisand) and plot the distributions of the cosine distance between each face frame and the video feature centre in Figure \ref{fig:costheta}. As evident from the plots, the mean of PCPSFM is the highest, which indicates that the proposed method is able to remove the noise of the each video frame and generate new frontal faces with low intra-variance. These findings are in agreement with our recognition results above and confirm the advantages of combing GAN and robust PCA for the application of face recognition.

\begin{figure}[t]
\centering
\includegraphics[width=0.8\linewidth]{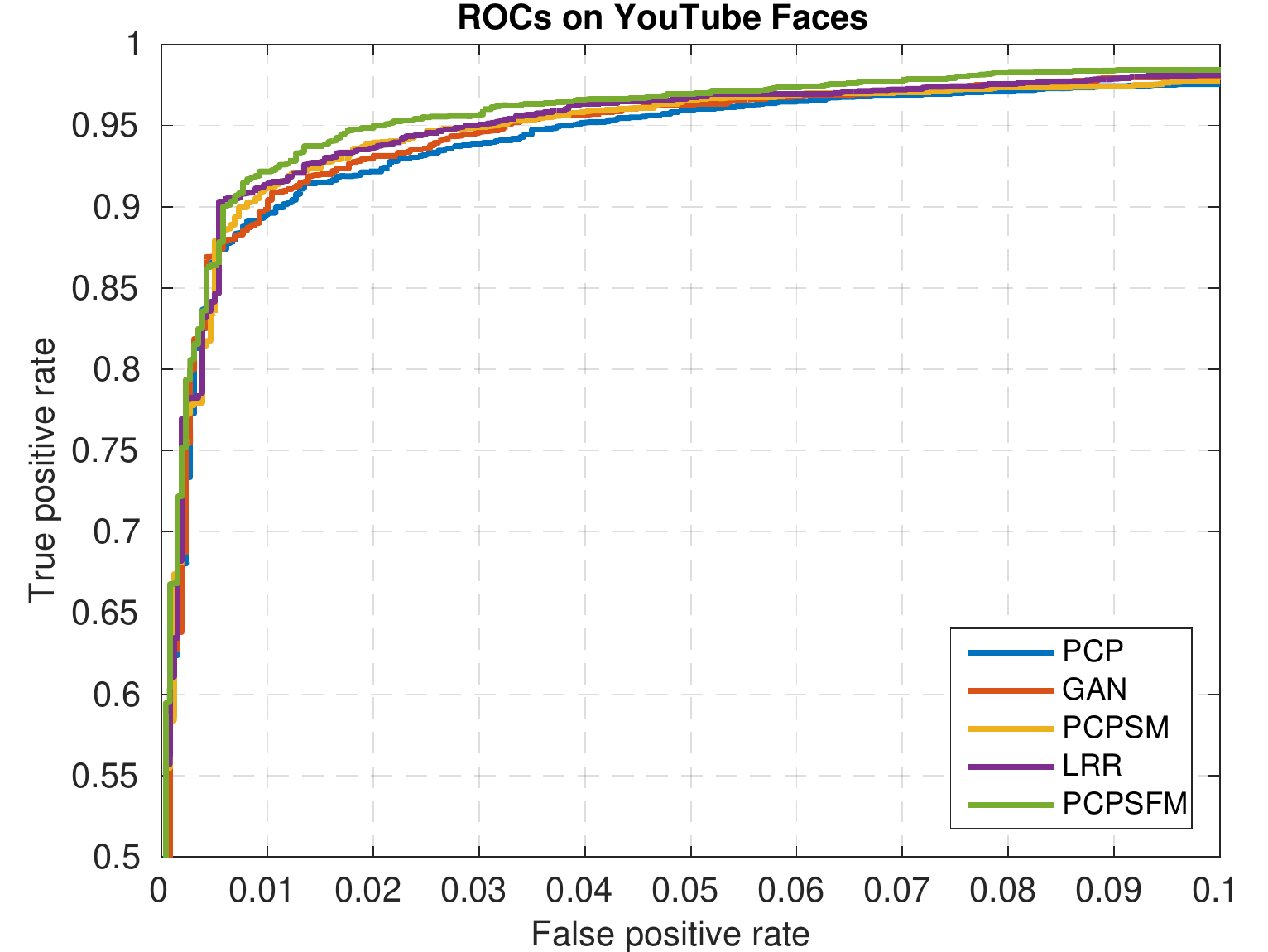}
\caption{ROC curves of the proposed methods on the YouTube Faces database under the ``restricted'' protocol.}
\label{fig:ytf_rocs}
\end{figure}

\section{Conclusion}
In this paper, we study the problem of robust principal component analysis with features acting as side information in the presence missing values. For the application domain of UV completion, we also propose the use of generative adversarial networks to extract side information and subspaces, which, to the best of our knowledge, is the first occasion  where RPCA and GAN have been fused. We also prove the convergence of ADMM for our convex objective. Through synthetic and real-world experiments, we demonstrate the advantages of side information. In virtue of in-the-wild data, we corroborate our fusion strategy. Finally, face recognition benchmarks accredit the efficacy of our proposed approach over state-of-art methods.  

\section{Supplementary Material}
\subsection{Generative adversarial networks}

For GAN, we employ the image-to-image conditional adversarial network\cite{isola2016image} to conduct UV completion. As is shown in Figure~\ref{fig:UVGAN_framework}, there are two main components in the image-to-image conditional GAN: a generator module and a discriminator module.

\begin{figure}[h]
\centering
\includegraphics[width=0.8\linewidth]{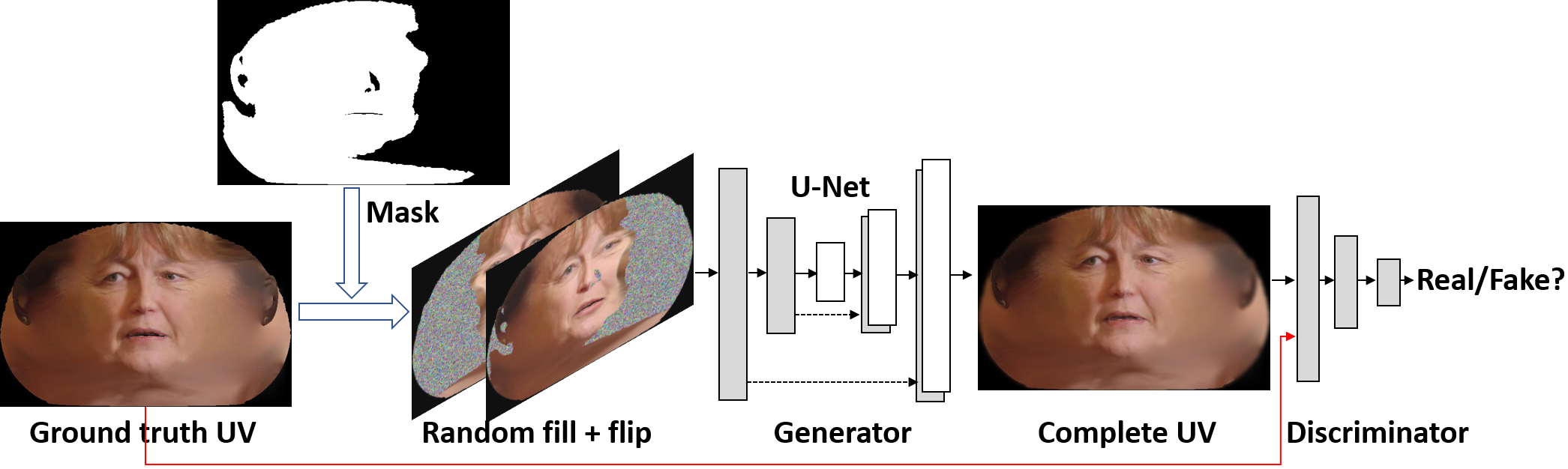}
\caption{Our GAN framework consists of one generator and one discriminator. The generator takes the incomplete UV map as input and outputs the full UV map. The discriminator is learnt to validate the genuineness of the synthesised UV texture. Note that only the generator is used at the testing stage.}
\label{fig:UVGAN_framework}
\end{figure}

{\bf Generator Module} Given incomplete UV texture input, the generator $G$ works as an auto-encoder to construct completed instances. We adopt the pixel-wise $l_{1}$ norm as the reconstruction loss:
\vspace{-2mm}
\begin{equation}
L_{gen}=\frac{1}{W\times H}\sum_{i=1}^{W}\sum_{j=1}^{H}\left | I_{i,j}-I_{i,j}^{*} \right |,
\label{eq:genloss}
\end{equation}
where $I_{i,j}$ is the estimated UV texture and $I_{i,j}^{*}$ is the ground truth texture of width $W$ and height $H$. To preserve the image information in the original resolution, we follow the encoder-decoder design in~\cite{isola2016image}, where skip connections between mirrored layers in the encoder and decoder stacks are made. We first fill the incomplete UV texture with random noise and then concatenate it with its mirror image as the generator input. Since the face is not exactly symmetric, we have avoided using symmetry loss as in\cite{Huang_2017_ICCV}. Also, unlike the original GAN model\cite{goodfellow2014generative} which is initialised from a noise vector, the hidden representations obtained from our encoder capture more variations as well as relationships between invisible and visible regions, and thus help the decoder fill up the missing regions.

{\bf Discriminator Module} Although the previous generator module can fill missing pixels with small reconstruction errors, it does not guarantee the output textures to be visually realistic and informative. With only the pixel-wise $l_{1}$ reconstruction loss, the UV completion results would be quite blurry and missing important details. To improve the quality of synthetic images and encourage more photo-realistic results, we adopt a discriminator module $D$ to distinguish real and fake UVs. The adversarial loss, which is a reflection of how the generator could maximally fool the discriminator and how well the discriminator could distinguish between real and fake UVs, is defined as
\vspace{-2mm}
\begin{align}
L_{adv} = \min_G\max_D\ &\mathbb{E}_{\mathbf{x}\sim p_{d}(\mathbf{x}),\mathbf{y}\sim p_{d}(\mathbf{y})}\left [ \log D(\mathbf{x,y}) \right ] + \notag\\ 
&\mathbb{E}_{\mathbf{x}\sim p_{d}(\mathbf{x}),\mathbf{z}\sim p_{z}(\mathbf{z})}\left [\log( 1 - D(\mathbf{x},G(\mathbf{x,z}))) \right ],
\label{eq:adver}
\end{align}
where $ p_{z}(\mathbf{z})$, $p_{d}(\mathbf{x})$ and $p_{d}(\mathbf{y})$ represent the distributions (Gaussian) of the noise variable $\mathbf{z}$, the partial UV texture $\mathbf{x}$ and the full UV texture $\mathbf{y}$ respectively.

{\bf Objective Function} The final loss function for the proposed UV-GAN is a weighted sum of generator loss and discriminator loss:
\vspace{-2mm}
\begin{equation}
L=L_{gen}+ \lambda L_{adv\_g}.
\end{equation}\label{eq:loss}
$\lambda$ is the weight to balance generator loss and discriminator loss. $\lambda$ is set to $0.01$ empirically.

{\bf Architecture} The same network architecture of \cite{isola2016image} is adopted here\footnote{https://github.com/phillipi/pix2pix}. The encoder unit consists of convolution, batch normalisation and ReLU, and the decoder unit consists of deconvolution, batch normalisation and ReLU. The convolution involves $4\times4$ spatial filters applied with stride $2$. Convolution in the encoder and the discriminator is also downsampled by a factor of $2$, while in the decoder it is upsampled by a factor of $2$. 

As shown in Figure \ref{fig:subfig:Generator}, the generator utilises the U-Net~\cite{ronneberger2015u} architecture which has skip connections between $i^{th}$ layer in the encoder and the $(n-i)^{th}$ layer in the decoder, where $n$ is the total number of layers. These skip connections concatenate activations from the $i^{th}$ layer to the $(n-i)^{th}$ layer. Note that batch normalisation is not applied to the first Conv64 layer in the encoder. All ReLUs in the encoder are leaky, with slope 0.2, whereas ReLUs in the decoder are not leaky.

For the discriminator, we use the $70\times70$ PatchGAN as in \cite{isola2016image}. In Figure \ref{fig:subfig:Discriminator}, we depict the architecture of the discriminator. Again, batch normalisation is not applied to the first Conv64 layer. However, all ReLUs are now leaky, with slope 0.2. We have also set the stride of the last two encoder modules to $1$. 

\begin{figure}[h!]
  \centering
  \subfigure[Generator architecture]{
    \label{fig:subfig:Generator} 
    \includegraphics[width=0.45\textwidth]{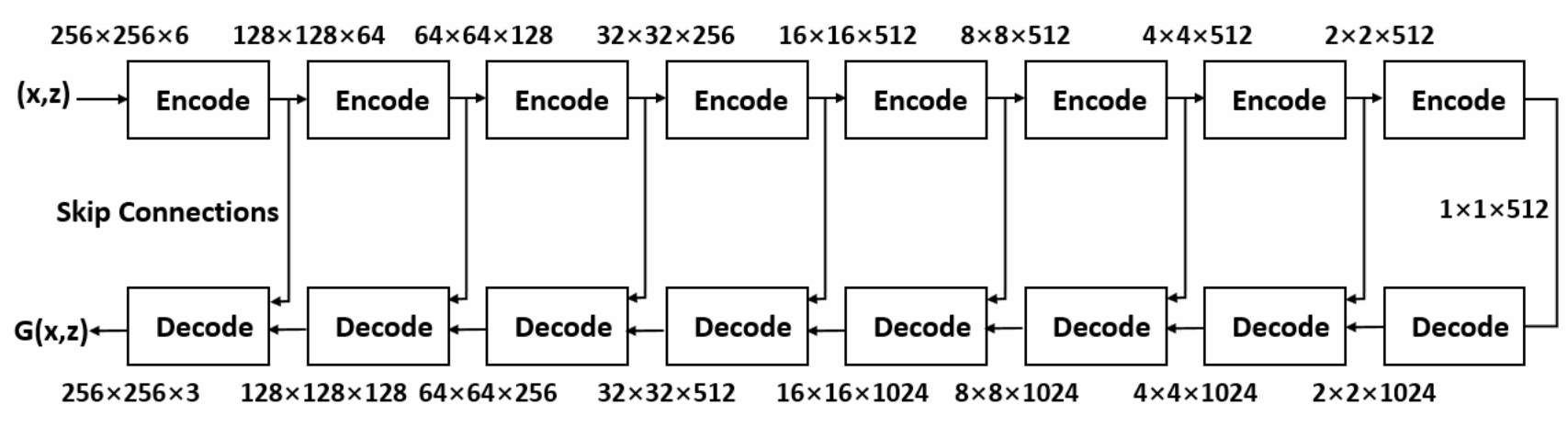}}
 \subfigure[Discriminator architecture]{
    \label{fig:subfig:Discriminator} 
    \includegraphics[width=0.45\textwidth]{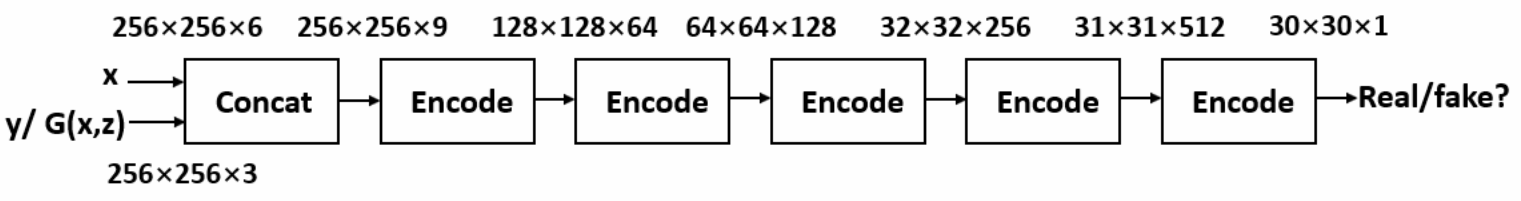}}
 \vspace{-2mm}
 \caption{The encoder unit consists of convolution, batch normalisation and ReLU, and the decoder unit consists of de-convolution, batch normalisation and ReLU. The input to the generator is the occluded UV map $x$ that is filled with random noise $z$ and concatenated with its flipped image. The input to the discriminator is the original input $x$ and either the ground-truth UV map $y$ or the generated UV map $G(x,z)$.}
\label{fig:uvgan}
\end{figure}

{\bf Training}
We train our networks from scratch by initialising the weights from a Gaussian distribution with zero mean and $0.02$ standard deviation. In order to train our UV completion model by pair-wise image data, we make use of both under-controlled and in-the-wild UV datasets. For the under-controlled UV data, we randomly select 180 subjects from the 4DFAB dataset\cite{cheng4dfab}. For the in-the-wild UV data, we employ the pseudo-complete UVs from the UMD video dataset\cite{bansal2017dosanddonts} via Poisson blending\cite{perez2003poisson}. We have meticulously chosen videos with large pose variations such that coverage of different poses is adequate. In the end, we have a combined UV dataset of 1,892 identities with 5,638 unique UV maps.

\subsection{Deep face feature embedding networks}
As shown in Figure~\ref{pic:res27}, we use ResNet-27~\cite{zhang2016range,wen2016discriminative} for $512$-$d$ facial feature embedding with softmax loss. The sizes of all the convolutional filters are $3\times3$ with stride $1$. And we set the kernel size of max-pooling to $2\times2$ with stride $2$. The network is initialised from Gaussian distribution and trained on the VGG training set (c3.1 million images) under the supervisory signals of soft-max. After an initial learning rate of 0.1, we successively contract it by a factor of 10 at the $6^{th}$, $14^{th}$, $22^{th}$ and $30^{th}$ epoch. We train the network in parallel on four GPUs so the overall batch size is $128\times4$. The input face size of the network is $112\times112$ pixels.

\begin{figure*}[hb!]
\begin{center}
    \includegraphics[width=\textwidth]{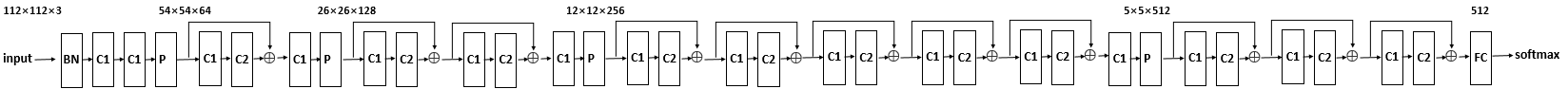}\\
\end{center}
\caption{Residual Network structure for facial feature embedding. C1 consists of convolution, batch normalisation and ReLU. C2 consists of convolution and batch normalisation. BN is the batch normalisation layer. P is the max-pooling layer. And FC is the fully connected layer.}
\label{pic:res27}
\end{figure*}


%

\ifCLASSOPTIONcaptionsoff
  \newpage
\fi




{
\bibliographystyle{IEEEtran}
\bibliography{egbib}

\begin{thebibliography}{10}
\providecommand{\url}[1]{#1}
\csname url@samestyle\endcsname
\providecommand{\newblock}{\relax}
\providecommand{\bibinfo}[2]{#2}
\providecommand{\BIBentrySTDinterwordspacing}{\spaceskip=0pt\relax}
\providecommand{\BIBentryALTinterwordstretchfactor}{4}
\providecommand{\BIBentryALTinterwordspacing}{\spaceskip=\fontdimen2\font plus
\BIBentryALTinterwordstretchfactor\fontdimen3\font minus
  \fontdimen4\font\relax}
\providecommand{\BIBforeignlanguage}[2]{{%
\expandafter\ifx\csname l@#1\endcsname\relax
\typeout{** WARNING: IEEEtran.bst: No hyphenation pattern has been}%
\typeout{** loaded for the language `#1'. Using the pattern for}%
\typeout{** the default language instead.}%
\else
\language=\csname l@#1\endcsname
\fi
#2}}
\providecommand{\BIBdecl}{\relax}
\BIBdecl

\bibitem{Booth14}
J.~Booth and S.~Zafeiriou, ``Optimal uv spaces for facial morphable model
  construction,'' in \emph{ICIP}, 2014.

\bibitem{3dmm_vetter}
V.~Blanz and T.~Vetter, ``A morphable model for the synthesis of 3d faces,'' in
  \emph{SIGGRAPH}, 1999.

\bibitem{3dmm_revisited_PatelS09}
A.~Patel and W.~A.~P. Smith, ``3d morphable face models revisited.'' in
  \emph{CVPR}, 1999.

\bibitem{Booth17}
J.~Booth, E.~Antonakos, S.~Ploumpis, G.~Trigeorgis, Y.~Panagakis, and
  S.~Zafeiriou, ``3d face morphable models in-the-wild,'' in \emph{CVPR}, 2017.

\bibitem{Cao17}
Q.~Cao, L.~Shen, W.~Xie, O.~Parkhi, and A.~Zisserman, ``Vggface2: A dataset for
  recognising faces across pose andage,'' in \emph{arXiv:1710.08092}, 2017.

\bibitem{Shang14}
F.~Shang, Y.~Liu, J.~Cheng, and H.~Cheng, ``Robust principal component analysis
  with missing data,'' in \emph{CIKM}, 2014, pp. 1149--1158.

\bibitem{Candes11}
E.~Cand{\`e}s, X.~Li, Y.~Ma, and J.~Wright, ``Robust principal component
  analysis?'' \emph{Journal of the ACM}, vol.~58, no.~3, pp. 1--37, 2011.

\bibitem{Chandrasekaran11}
V.~Chandrasekaran, S.~Sanghavi, P.~Parrilo, and A.~Willsky, ``Rank-sparsity
  incoherence for matrix decomposition,'' \emph{SIAM Journal on Optimization},
  vol.~21, no.~2, pp. 572–--596, 2011.

\bibitem{Aravkin14}
A.~Aravkin, S.~Becker, V.~Cevher, and P.~Olsen, ``A variational approach to
  stable principal component pursuit,'' in \emph{UAI}, 2014, pp. 32--41.

\bibitem{Bao12}
B.~Bao, G.~Liu, C.~Xu, and S.~Yan, ``Inductive robust principal component
  analysis,'' \emph{IEEE Transactions on Image Processing}, vol.~21, no.~8, pp.
  3794 -- 3800, 2012.

\bibitem{Cabral2013}
R.~Cabral, F.~{De la Torre}, J.~Costeira, and A.~Bernardino, ``Unifying nuclear
  norm and bilinear factorization approaches for low-rank matrix
  decomposition,'' in \emph{ICCV}, 2013.

\bibitem{Xu12}
H.~Xu, C.~Caramanis, and S.~Sanghavi, ``Robust pca via outlier pursuit,''
  \emph{IEEE Transactions on Information Theory}, vol.~58, no.~5, pp.
  3047--3064, 2012.

\bibitem{Zhou10}
Z.~Zhou, X.~Li, J.~Wright, E.~Cand{\`e}s, and Y.~Ma, ``Stable principal
  component pursuit,'' in \emph{ISIT}, 2010.

\bibitem{Jiao15}
J.~Jiao, T.~Courtade, K.~Venkat, and T.~Weissman, ``Justification of
  logarithmic loss via the benefit of side information,'' \emph{IEEE
  Transactions on Information Theory}, vol.~61, no.~10, pp. 5357--5365, 2015.

\bibitem{Ziv76}
A.~Wyner and J.~Ziv, ``The rate-distortion function for source coding with side
  information at the decoder,'' \emph{IEEE Transactions on Information Theory},
  vol.~22, no.~1, pp. 1--10, 1976.

\bibitem{Candes08}
E.~Cand{\`e}s, ``The restricted isometry property and its implications for
  compressed sensing,'' \emph{Comptes Rendus Mathematique}, vol. 346, no.~9,
  pp. 589--592, 2008.

\bibitem{Chiang15}
K.~Chiang, C.~Hsieh, and I.~Dhillon, ``Matrix completion with noisy side
  information,'' in \emph{NIPS}, 2015.

\bibitem{Xu13}
M.~Xu, J.~R, and Z.~Zhou, ``Speedup matrix completion with side information:
  Application to multi-label learning,'' in \emph{NIPS}, 2013.

\bibitem{Mota17}
J.~Mota, N.~Deligiannis, and M.~Rodrigues, ``Compressed sensing with prior
  information: Strategies, geometry, and bounds,'' \emph{IEEE Transactions on
  Information Theory}, 2017.

\bibitem{KaiYang16}
K.~Chiang, C.~Hsieh, and I.~Dhillon, ``Robust principal component analysis with
  side information,'' in \emph{ICML}, 2016.

\bibitem{Sagonas14}
C.~Sagonas, Y.~Panagakis, S.~Zafeiriou, and M.~Pantic, ``Raps: Robust and
  efficient automatic construction of person-specific deformable models,'' in
  \emph{CVPR}, 2014.

\bibitem{liu10}
G.~Liu, Z.~Lin, and Y.~Yu, ``Robust subspace segmentation by low-rank
  representation,'' in \emph{ICML}, 2010, pp. 663--670.

\bibitem{liu17}
G.~Liu, Q.~Liu, and P.~Li, ``Blessing of dimensionality: Recovering mixture
  data via dictionary pursuit,'' \emph{TPAMI}, vol.~39, no.~1, pp. 47--60,
  2017.

\bibitem{pathak2016}
D.~Pathak, P.~Krahenbuhl, J.~Donahue, T.~Darrell, and A.~Efros, ``Context
  encoders: Feature learning by inpainting,'' in \emph{CVPR}, 2016.

\bibitem{Yang2017}
C.~Yang, X.~Lu, Z.~Lin, E.~Shechtman, O.~Wang, and H.~Li, ``High-resolution
  image inpainting using multi-scale neural patch synthesis,'' in \emph{CVPR},
  2017.

\bibitem{Li2017}
Y.~Li, S.~Liu, J.~Yang, and M.~Yang, ``Generative face completion,'' in
  \emph{CVPR}, 2017.

\bibitem{Xue17}
N.~Xue, Y.~Panagakis, and S.~Zafeiriou, ``Side information in robust principal
  component analysis: Algorithms and applications,'' in \emph{ICCV}, 2017.

\bibitem{Chen13}
Y.~Chen, A.~Jalali, S.~Sanghavi, and C.~Caramanis, ``Low-rank matrix recovery
  from errors and erasures,'' \emph{IEEE Transactions on Information Theory},
  vol.~59, no.~7, pp. 4324--4337, 2013.

\bibitem{Toh2010}
K.-C. Toh and S.~Yun, ``An accelerated proximal gradient algorithm for nuclear
  norm regularized linear least squares problems,'' \emph{Pacific Journal of
  Optimization}, vol.~6, no. 615-640, p.~15, 2010.

\bibitem{Rockafellar76}
R.~T. Rockafellar, ``Monotone operators and the proximal point algorithm,''
  \emph{SIAM journal on control and optimization}, vol.~14, no.~5, pp.
  877--898, 1976.

\bibitem{Boyd11}
S.~Boyd, N.~Parikh, E.~Chu, B.~Peleato, and J.~Eckstein, ``Distributed
  optimization and statistical learning via the alternating direction method of
  multipliers,'' \emph{Foundations and Trends in Machine Learning}, vol.~3,
  no.~1, pp. 1--122, 2011.

\bibitem{lin09}
Z.~lin, M.~Chen, and Y.~Ma, ``The augmented lagrange multiplier method for
  exact recovery of corrupted low-rank matrices,'' \emph{UIUC Technical
  Report}, 2009.

\bibitem{Sun16}
H.~Sun, J.~Wang, and T.~Deng, ``On the global and linear convergence of direct
  extension of {ADMM} for 3-block separable convex minimization models,''
  \emph{Journal of Inequalities and Applications}, no. 227, p. 227, 2016.

\bibitem{Hintermüller15}
M.~Hinterm{\"u}ller and T.~Wu, ``Robust principal component pursuit via inexact
  alternating minimization on matrix manifolds,'' \emph{Journal of Mathematical
  Imaging and Vision}, vol.~51, no.~3, pp. 361--377, 2015.

\bibitem{Oh16}
T.~Oh, Y.~Tai, J.~Bazin, H.~Kim, and I.~Kweon, ``Partial sum minimization of
  singular values in robust {PCA}: Algorithm and applications,'' \emph{TPAMI},
  vol.~38, no.~4, pp. 744--758, 2016.

\bibitem{Shabalin13}
A.~Shabalin and A.~Nobel, ``Reconstruction of a low-rank matrix in the presence
  of gaussian noise,'' \emph{Journal of Multivariate Analysis}, vol. 118, pp.
  67--76, 2013.

\bibitem{Basri03}
R.~Basri and D.~Jacobs, ``Lambertian reflectance and linear subspaces,''
  \emph{TPAMI}, vol.~25, no.~2, pp. 218--233, 2003.

\bibitem{Patel12}
V.~Patel, T.~Wu, S.~Biswas, P.~Phillips, and R.~Chellappa, ``Dictionary-based
  face recognition under variable lighting and pose,'' \emph{IEEE Transactions
  on Information Forensics and Security}, vol.~7, no.~3, pp. 954--965, 2012.

\bibitem{Aharon06}
M.~Aharon, M.~Elad, and A.~Bruckstein, ``$ rm k $-svd: An algorithm for
  designing overcomplete dictionaries for sparse representation,'' \emph{IEEE
  Transactions on signal processing}, vol.~54, no.~11, pp. 4311--4322, 2006.

\bibitem{Blanz03}
V.~Blanz and T.~Vetter, ``Face recognition based on fitting a 3d morphable
  model,'' \emph{TPAMI}, vol.~25, no.~9, pp. 1063--1074, 2003.

\bibitem{Booth_2016_CVPR}
J.~Booth, A.~Roussos, S.~Zafeiriou, A.~Ponniah, and D.~Dunaway, ``A 3d
  morphable model learnt from 10,000 faces,'' in \emph{CVPR}, 2016.

\bibitem{booth20173d}
J.~Booth, E.~Antonakos, S.~Ploumpis, G.~Trigeorgis, Y.~Panagakis, and
  S.~Zafeiriou, ``3d face morphable models in-the-wild,'' in \emph{CVPR}, 2017.

\bibitem{bulat2017far}
A.~Bulat and G.~Tzimiropoulos, ``How far are we from solving the 2d and 3d face
  alignment problem?(and a dataset of 230,000 3d facial landmarks),'' in
  \emph{ICCV}, 2017.

\bibitem{cheng4dfab}
S.~Cheng, I.~Kotsia, M.~Pantic, and S.~Zafeiriou, ``4dfab: A large scale 4d
  facial expression database for biometric applications,'' in
  \emph{arXiv:1712.01443}, 2017.

\bibitem{isola2016image}
P.~Isola, J.-Y. Zhu, T.~Zhou, and A.~A. Efros, ``Image-to-image translation
  with conditional adversarial networks,'' in \emph{arXiv:1611.07004}, 2016.

\bibitem{300vw}
J.~Shen, S.~Zafeiriou, G.~G. Chrysos, J.~Kossaifi, G.~Tzimiropoulos, and
  M.~Pantic, ``The first facial landmark tracking in-the-wild challenge:
  Benchmark and results,'' in \emph{ICCVW}, 2015, pp. 1003--1011.

\bibitem{taigman2014deepface}
Y.~Taigman, M.~Yang, M.~Ranzato, and L.~Wolf, ``Deepface: Closing the gap to
  human-level performance in face verification,'' in \emph{CVPR}, 2014.

\bibitem{sun2014deep}
Y.~Sun, Y.~Chen, X.~Wang, and X.~Tang, ``Deep learning face representation by
  joint identification-verification,'' in \emph{NIPS}, 2014, pp. 1988--1996.

\bibitem{schroff2015facenet}
F.~Schroff, D.~Kalenichenko, and J.~Philbin, ``Facenet: A unified embedding for
  face recognition and clustering,'' in \emph{CVPR}, 2015.

\bibitem{parkhi2015deep}
O.~M. Parkhi, A.~Vedaldi, and A.~Zisserman, ``Deep face recognition.'' in
  \emph{BMVC}, vol.~1, no.~3, 2015, p.~6.

\bibitem{chen2017robust}
J.~Chen, V.~Patel, L.~Liu, V.~Kellokumpu, G.~Zhao, M.~Pietik{\"a}inen, and
  R.~Chellappa, ``Robust local features for remote face recognition,''
  \emph{Image and Vision Computing}, 2017.

\bibitem{ranjan2017l2}
R.~Ranjan, C.~D. Castillo, and R.~Chellappa, ``L2-constrained softmax loss for
  discriminative face verification,'' in \emph{arXiv:1703.09507}, 2017.

\bibitem{shekhar2017synthesis}
S.~Shekhar, V.~M. Patel, and R.~Chellappa, ``Synthesis-based robust low
  resolution face recognition,'' in \emph{arXiv:1707.02733}, 2017.

\bibitem{Ding16}
C.~Ding and D.~Tao, ``A comprehensive survey on pose-invariant face
  recognition,'' \emph{ACM Transactions on Intelligent Systems and Technology
  (TIST)}, vol.~7, no.~3, pp. 1--42, 2016.

\bibitem{Sharma12}
A.~Sharma, A.~Kumar, H.~Daume, and D.~Jacobs, ``Generalized multiview analysis:
  A discriminative latent space,'' in \emph{CVPR}, 2012.

\bibitem{Ding15}
C.~Ding, J.~Choi, D.~Tao, and L.~Davis, ``Multi-directional multi-level
  dual-cross patterns for robust face recognition,'' \emph{TPAMI}, 2015.

\bibitem{Kan14}
M.~Kan, S.~Shan, H.~Chang, and X.~Chen, ``Stacked progressive auto-encoders
  (spae) for face recognition,'' in \emph{CVPR}, 2014.

\bibitem{cao2017vggface2}
Q.~Cao, L.~Shen, W.~Xie, O.~M. Parkhi, and A.~Zisserman, ``Vggface2: A dataset
  for recognising faces across pose and age,'' in \emph{arXiv:1710.08092},
  2017.

\bibitem{sengupta2016frontal}
S.~Sengupta, J.-C. Chen, C.~Castillo, V.~M. Patel, R.~Chellappa, and D.~W.
  Jacobs, ``Frontal to profile face verification in the wild,'' in \emph{WACV},
  2016.

\bibitem{wolf2011face}
L.~Wolf, T.~Hassner, and I.~Maoz, ``Face recognition in unconstrained videos
  with matched background similarity,'' in \emph{CVPR}, 2011.

\bibitem{zhang2016range}
X.~Zhang, Z.~Fang, Y.~Wen, Z.~Li, and Y.~Qiao, ``Range loss for deep face
  recognition with long-tail,'' in \emph{ICCV}, 2017.

\bibitem{wen2016discriminative}
Y.~Wen, K.~Zhang, Z.~Li, and Y.~Qiao, ``A discriminative feature learning
  approach for deep face recognition,'' in \emph{ECCV}, 2016.

\bibitem{liu2017sphereface}
W.~Liu, Y.~Wen, Z.~Yu, M.~Li, B.~Raj, and L.~Song, ``Sphereface: Deep
  hypersphere embedding for face recognition,'' in \emph{CVPR}, 2017.

\bibitem{deng2017marginal}
J.~Deng, Y.~Zhou, and S.~Zafeiriou, ``Marginal loss for deep face
  recognition,'' in \emph{CVPRW}, 2017.

\bibitem{Huang_2017_ICCV}
R.~Huang, S.~Zhang, T.~Li, and R.~He, ``Beyond face rotation: Global and local
  perception gan for photorealistic and identity preserving frontal view
  synthesis,'' in \emph{ICCV}, 2017.

\bibitem{goodfellow2014generative}
I.~Goodfellow, J.~Pouget-Abadie, M.~Mirza, B.~Xu, D.~Warde-Farley, S.~Ozair,
  A.~Courville, and Y.~Bengio, ``Generative adversarial nets,'' in \emph{NIPS},
  2014, pp. 2672--2680.

\bibitem{ronneberger2015u}
O.~Ronneberger, P.~Fischer, and T.~Brox, ``U-net: Convolutional networks for
  biomedical image segmentation,'' in \emph{MICCAI}, 2015.

\bibitem{bansal2017dosanddonts}
A.~Bansal, C.~Castillo, R.~Ranjan, and R.~Chellappa, ``The do's and don'ts for
  cnn-based face verification,'' in \emph{arXiv:1705.07426}, 2017.

\bibitem{perez2003poisson}
P.~P{\'e}rez, M.~Gangnet, and A.~Blake, ``Poisson image editing,'' in
  \emph{TOG}, 2003.

\end{thebibliography}
}

\end{document}